\documentclass{article}

\usepackage{microtype}
\usepackage{graphicx}
\usepackage{subfigure}
\usepackage{booktabs} 

\usepackage{hyperref}
\usepackage{hyperref}
\usepackage{hyperref}
\usepackage{url}
\usepackage{amsmath,amssymb,graphicx}
\usepackage{stfloats}
\usepackage{amsthm}
\usepackage{multirow}
\usepackage{diagbox}
\newcommand{\cdummy}{\cdot}
\newcommand{\mathd}{\mathrm{d}}

\newcommand{\tmop}[1]{\ensuremath{\operatorname{#1}}}
\newcommand{\tmtextbf}[1]{\text{{\bfseries{#1}}}}

\DeclareMathOperator*{\margmax}{\arg\max}

\makeatletter
\newtheorem*{rep@theorem}{\rep@title}
\newcommand{\newreptheorem}[2]{%
\newenvironment{rep#1}[1]{%
 \def\rep@title{#2 \ref{##1}}%
 \begin{rep@theorem}}%
 {\end{rep@theorem}}}
\makeatother

\newtheorem{corollary}{Corollary}
\newtheorem{theorem}{Theorem}
\newreptheorem{theorem}{Theorem}

\usepackage[accepted]{icml2022}

\icmltitlerunning{GSmooth: Certified Robustness against Semantic Transformations via Generalized Randomized Smoothing}

\begin{document}

\newcommand{\fix}{\marginpar{FIX}}
\newcommand{\new}{\marginpar{NEW}}

\definecolor{mydarkblue}{rgb}{0,0.08,0.45}
\definecolor{mydarkgreen}{RGB}{0, 139, 69}
\hypersetup{
	colorlinks=true,
	urlcolor=magenta,
	citecolor=mydarkblue,
}
\definecolor{mycyan}{cmyk}{.3,0,0,0}

\newcommand{\revise}[1]{{\color{blue}{#1}}}
\newcommand{\yinpeng}[1]{{\color{red}{[yinpeng:#1]}}}
\newcommand{\junz}[1]{{\color{blue}{[jz:#1]}}}
\newcommand{\hangx}[1]{{\color{red}{[Hang:#1]}}}
\newcommand{\zhongkai}[1]{{\color{yellow}{[Z:#1]}}}
\newcommand{\hang}[1]{{\color{blue}{[Hang: #1]}}}
\newcommand{\ycy}[1]{{\color{red}{[ycy:#1]}}}

\twocolumn[
\icmltitle{GSmooth: Certified Robustness against \\ Semantic Transformations via Generalized Randomized Smoothing}



\begin{icmlauthorlist}
\icmlauthor{Zhongkai Hao}{to,ee}
\icmlauthor{Chengyang Ying}{to}
\icmlauthor{Yinpeng Dong}{to,re}
\icmlauthor{Hang Su}{to,goo,ed}
\icmlauthor{Jian Song}{ee,tv}
\icmlauthor{Jun Zhu}{to,re,goo,ed}
\end{icmlauthorlist}

\icmlaffiliation{to}{Dept. of Comp. Sci. \& Techn., Institute for AI, BNRist Center, Tsinghua-Bosch Joint ML Center, Tsinghua University}
\icmlaffiliation{goo}{Peng Cheng Laboratory}
\icmlaffiliation{ed}{Tsinghua University-China Mobile
Communications Group Co., Ltd. Joint Institute}
\icmlaffiliation{ee}{Dept. of EE, Tsinghua University}
\icmlaffiliation{re}{RealAI}
\icmlaffiliation{tv}{Key Laboratory of Digital TV System of Guangdong Province and Shenzhen City, Research Institute of Tsinghua University in Shenzhen}
\icmlcorrespondingauthor{Hang Su}{suhangss@tsinghua.edu.cn}
\icmlcorrespondingauthor{Jun Zhu}{dcszj@tsinghua.edu.cn}

\icmlkeywords{Machine Learning, ICML}

\vskip 0.3in
]



\printAffiliationsAndNotice{} 

\begin{abstract}
Certified defenses such as randomized smoothing have shown promise towards building reliable machine learning systems against $\ell_p$-norm bounded attacks. However, existing methods are insufficient or unable to provably defend against semantic transformations, especially those without closed-form expressions (such as defocus blur and pixelate), which are more common in practice and often unrestricted. 
To fill up this gap, we propose generalized randomized smoothing (GSmooth), a unified theoretical framework for certifying robustness against general semantic transformations via a novel dimension augmentation strategy. Under the GSmooth framework, we present a scalable algorithm that uses a surrogate image-to-image network to approximate the complex transformation. The surrogate model provides a powerful tool for studying the properties of semantic transformations and certifying robustness. Experimental results on several datasets demonstrate the effectiveness of our approach for robustness certification against multiple kinds of semantic transformations and corruptions, which is not achievable by the alternative baselines.
\end{abstract}

\section{Introduction}





Deep learning models are vulnerable to adversarial examples \citep{biggio2013evasion,szegedy2013intriguing,goodfellow2014explaining,dong2018boosting} as well as semantic transformations \citep{engstrom2019exploring,hendrycks2019benchmarking}, which can limit their applications in various security-sensitive tasks. For example, a small adversarial patch on the road markings can mislead the autonomous driving system \citep{jing2021too,qiaoben2021understanding}, which raises severe safety concerns.
Compared with the $\ell_p$-norm bounded adversarial examples, semantic transformations can occur more naturally in real-world scenarios and are often unrestricted, including image rotation, translation, blur, weather, etc., most of which are common corruptions \citep{hendrycks2019benchmarking}. Such transformations do not damage the semantic features of images that can still be recognized by humans, but they degrade the performance of deep learning models significantly. 
Therefore, it is imperative to improve model robustness against these semantic transformations.


\begin{figure}
    \centering
    \includegraphics[width=0.99\linewidth]{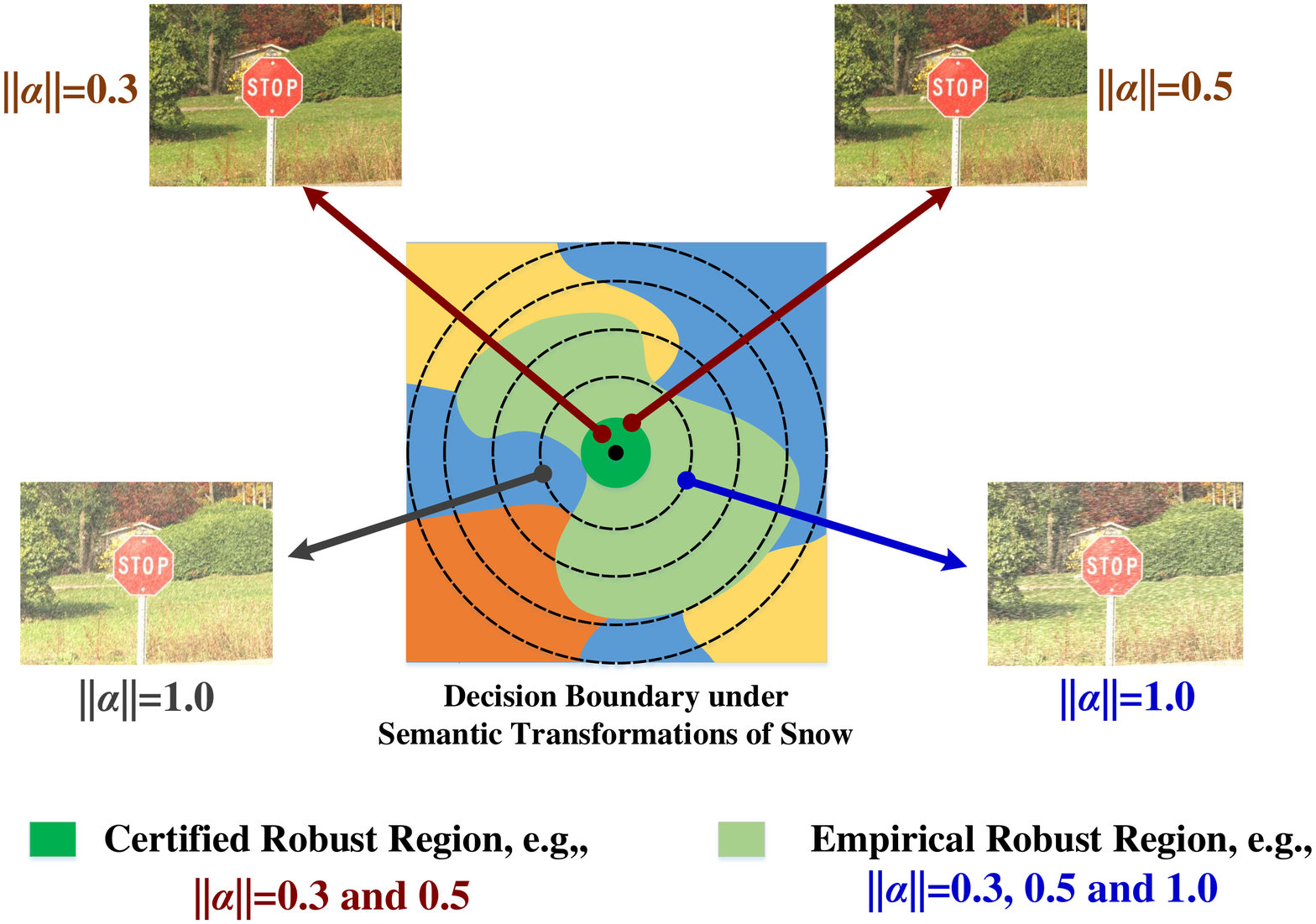}
    \caption{Illustration of certified defense against semantic transformations. We take snow as an example, where $\alpha$ is the parameter corresponding to the severity. The model is certifiably robust with $\|\alpha\|=0.3$ and $0.5$. Nevertheless, the model may make an erroneous prediction when $\|\alpha\|=1.0$ although it might be empirically robust to other corrupted inputs with $\|\alpha\|=1.0$. This implies that the empirical robustness may not be sufficient for safety-sensitive applications such as autonomous driving.}
    \label{fig1_}
\end{figure}

Although various methods \citep{wang2019implicit,hendrycks2019augmix,xie2020adversarial,calian2021defending} can empirically improve model robustness against semantic transformations on typical benchmarks (evaluated in average-case), these methods often fail to defend against adaptive attacks by generating adversarial semantic transformations \cite{hosseini2018semantic,engstrom2019exploring,wong2019wasserstein}, which are optimized over the parameter space of transformations for the worst-case. In contrast, the certified defenses aim to provide a certified region where the model is theoretically robust under any attack or perturbation \citep{wong2018provable,cohen2019certified, gowal2018effectiveness,zhang2019towards}. As an example in Fig.~\ref{fig1_}, we might encounter any levels and types of complex semantic transformations for autonomous driving cars in practice. An ideal safe model should tell users its safe regions that the model is certifiably robust under different transformations.

Several recent studies \citep{fischer2020certified,mohapatra2020towards,li2021tss,li2020provable,wu2022toward} have attempted to extend the certified defenses to simple semantic transformations with good mathematical properties like translation and Gaussian blur. However, 
these methods are neither scalable nor capable of certifying robustness against complex image corruptions and transformations, especially the non-resolvable ones \citep{li2021tss}. Typically, many non-resolvable transformations, such as zoom blur and pixelate, do not have closed-form expressions. This makes the theoretical analyses of these transformations difficult and sometimes impossible with the existing methods, although they are common in real-world scenarios. Therefore, it remains highly challenging to certify robustness against these complex and realistic semantic transformations.

In this paper, we propose \textbf{generalized randomized smoothing (GSmooth)}, a novel unified theoretical framework for certified robustness against general semantic transformations, including both the resolvable ones (e.g., translation) and non-resolvable ones (e.g., rotational blur).
Specifically, we propose to use an image-to-image translation neural network to approximate these transformations. 
Due to the strong capacity of neural networks, our method is flexible and scalable for modeling these complex non-resolvable semantic transformations. Then, we can theoretically provide the certified radius for the proxy neural networks by introducing new augmented noise in the layers of the surrogate model, which can be used for certifying the original transformations.
We further provide a theoretical analysis and an error bound for the approximation, which shows that the impact of the approximation error on the certified bound can be ignored in practice.
Finally, we validate the effectiveness of our methods on several publicly available datasets. The results demonstrate that our method is effective for certifying complex semantic transformations. 
Our GSmooth achieves state-of-the-art performance in both certified accuracy and empirical accuracy for different types of blur or image quality corruptions.

\section{Related Work}



\subsection{Robustness under Semantic Transformations}
Unlike an $\ell_p$ perturbation adding a small amount of noise to an image, semantic attacks are usually unrestricted. Representative works include the adversarial patches \cite{brown2017adversarial,song2018physical}, the manipulation based on spatial transformations such as rotation or translation \cite{xiao2018spatially,engstrom2019exploring}.
Recently, \citet{hendrycks2019benchmarking} show that a wide variety of image corruptions and perturbations degrade the performance of many deep learning models. Most of them such as types of blur, pixelate are hard to be analyzed mathematically, such that defending against them is highly challenging.  Various data augmentation techniques \citep{cubuk2019autoaugment,hendrycks2019augmix,hendrycks2021many, robey2020model} have been developed to enhance robustness under semantic perturbations. Also, \citet{calian2021defending} propose adversarial data augmentation that can be viewed as adversarial training for defending against semantic perturbations.

\subsection{Certified Defense against Semantic Transformations}
Beyond empirical defenses, several works \citep{singh2019abstract,balunovic2019certifying,mohapatra2020towards} have attempted to certify some simple geometric transformations. However, these are deterministic certification approaches and their performance on realistic datasets is unsatisfactory. Randomized smoothing is a novel certification method that originated from differential privacy \citep{lecuyer2019certified}. \citet{cohen2019certified} then improve the certified bound and apply it to realistic settings. \citet{yang2020randomized} exhaustively analyze the bound by
using different noise distributions and norms. \citet{hayes2020extensions} and \citet{yang2020randomized} point out that randomized smoothing suffers from the curse of dimensionality for the $l_{\infty}$ norm. 
\citet{fischer2020certified} and \citet{li2021tss} extend randomized smoothing to certify some simple semantic transformations, e.g., image translation and rotation. They show that randomized smoothing could be used to certify attacks beyond $l_p$ norm. However, their methods are limited to simple semantic transformations, which are easy to analyze due to their resolvable mathematical properties. Scalable algorithms for certifying most non-resolvable and complex semantic transformations remain unexplored.

\section{Proposed Method}
In this section, we present generalized randomized smoothing (GSmooth) with theoretical analyses. 
\subsection{Problem Formulation}
We first introduce the notations and problem formulation. Given the input $x\in\mathbb{R}^n$ and the label of $\mathcal{Y}= \{ 1, 2, \ldots, K \}$, we denote the classifier as $f (x) : \mathbb{R}^n \rightarrow [0,1]^K$, which outputs predicted probabilities over all $K$ classes. The prediction of $f$ is $\margmax_{i \in \mathcal{Y}} f(x)_i$, where $f(\cdot)_i$ denotes the $i$-th element of $f(\cdot)$. We denote $\tau (\theta,x) : \mathbb{R}^m \times \mathbb{R}^n \rightarrow \mathbb{R}^n$ to be a semantic
transformation of the raw input $x$ with parameter $\theta\in \mathbb{R}^m$. We define the smoothed classifier $G(x)$ as
\begin{equation}\label{1}
  G (x) =\mathbb{E}_{\theta \sim g(\cdot)} [f (\tau (\theta, x))],
\end{equation}
which is the average prediction for the samples under a smooth distribution $g(\theta) \propto \tmop{exp}(-\psi(\theta))$, here $\psi:\mathbb{R}^m \to \mathbb{R}$ is a smooth function. We let $y_A$ be the predicted label of the smoothed classifier $G(x)$ for a clean image, and we use $G(x)_A$ to denote the probability of the top-1 class $y_A$ in the rest of the paper as follows. Similarly, we can define $y_B$ as the runner-up class of the smoothed classifier $G (x)$. 
\begin{equation}
    y_A \triangleq \margmax_{i \in \mathcal{Y}} G(x)_i,\quad
    y_B \triangleq \margmax_{i \in \mathcal{Y} \backslash y_A} G(x)_i,
\end{equation}
and we use $G (x)_B$ to denote the probability of class $y_B$.

A classifier has a certified robust radius $R$ if it satisfies that for any perturbation $\|\xi\| \leqslant R$ where $\|\cdot\|$ is any $l_p$ norm without specification, we have
\begin{equation}
    \mathop{\arg\max}_{i \in \mathcal{Y}} G(\tau(\xi,x))_i = y_A.
\end{equation}
In this paper, we categorize semantic transformations into two classes: resolvable and non-resolvable. A semantic transformation is resolvable if the composition of two transformations
with parameters belonging to a perturbation set $\theta, \xi \in P \subset \mathbb{R}^m$, is still a transformation with
a new parameter  $\gamma = \gamma(\theta, \xi) \in P\subset \mathbb{R}^m$, here $\gamma (\cdummy, \cdummy):P\times P\rightarrow P$ is a function depending on these parameters, i.e.,
satisfying
\begin{equation}
    \tau(\theta, \tau(\xi, x)) = \tau(\gamma(\theta, \xi), x).
\end{equation}
Otherwise, it is non-resolvable, e.g., zoom blur and pixelate. The theoretical properties of resolvable transformations make it much easier to derive the certified bound. Therefore, in next two subsections, we will discuss resolvable and non-resolvable cases in turn under a unified framework.

\begin{figure*}
    \centering
    \includegraphics[width=0.99\textwidth]{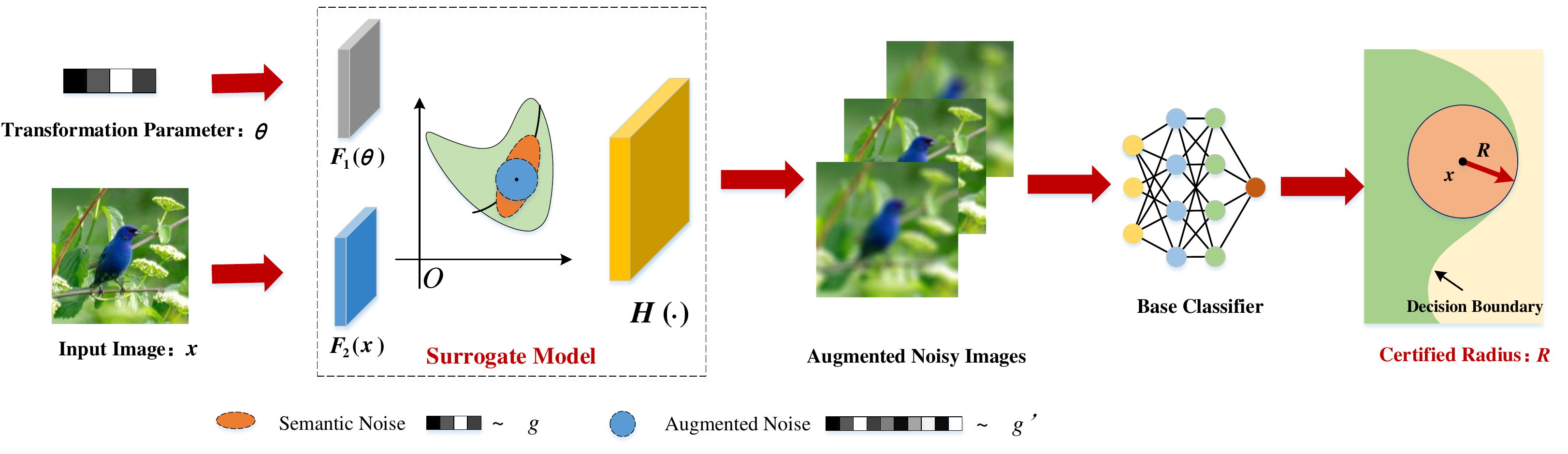}
    \caption{A graphical illustration of our GSmooth. We use a surrogate image-to-image translation network to accurately fit a semantic transformation. Then we add new augmented noise into the surrogate model and construct the GSmooth classifier. The augmented noise is sampled to ensure that the transformation is resolvable in the semantic space. We theoretically calculate the certified bound for the surrogate model to certify the original transformation.}
    \label{fig1}
\end{figure*}
\subsection{Certified Bound for Resolvable Transformations}
We first discuss a class of transformations that are resolvable. We introduce a function $\Phi(\cdot)$ which will be used in the certified bound. For any vector $u$ with unit norm, i.e., $\|u\|=1$, we set $\gamma_{u}\triangleq\langle u,\nabla \psi(\delta) \rangle$ as a random variable, where $\delta \sim g(\cdot)$ and $\nabla$ is the gradient operator.
We further define the complementary Cumulative Distribution Function (CDF) of $\gamma_u$ as $\varphi_u(c)=\mathbb{P}[\gamma_u > c]$ and the inverse complementary CDF of $\gamma_u$ as $\varphi^{-1}_u(p)=\tmop{inf}\{c|\mathbb{P}(\gamma_u>c)\leqslant p\}$. Following~\citet{yang2020randomized}, we define a function $\Phi$ as 
\begin{equation}
\label{eq2}
    \Phi(p) = \mathop{\max}_{\|u\|=1}\mathbb{E}[\gamma_u\mathbb{I}\{\gamma_u>\varphi^{-1}_u(p)\}].
\end{equation}
For resolvable semantic transformations, we have the following theorem. 
\begin{theorem}[\textbf{Certified bound for Resolvable Transformations}, proof in Appendix~\ref{app_thm_1}]
\label{thm-1}
For any classifier $f (x)$, let $G (x)$ be the smoothed classifier defined in Eq.~(\ref{1}). If there exists a function $M(\cdot,\cdot):P\times P\rightarrow \mathbb{R}^{m \times m}$ satisfying
  \[ \frac{\partial \gamma (\theta, \xi)}{\partial \xi} = \frac{\partial
     \gamma (\theta, \xi)}{\partial \theta} M (\theta, \xi), \]
  and there exist two constants $\underline{p_A}$, $\overline{p_B}$ satisfying
  \[ G (x)_A \geqslant \underline{p_A} \geqslant \overline{p_B} \geqslant G
     (x)_B, \]
  then $ \margmax_{i \in \mathcal{Y}} G (\tau (\xi, x))_i = y_A$ holds for any
  $\| \xi \| \leqslant R$, where
  \begin{equation}
    R = \frac{1}{ 2M^*} \int_{\overline{p_B}}^{\underline{p_A}} \frac{1}{\Phi(p)}\mathd p ,
  \end{equation}
and $M^* = \max_{\xi,\theta \in P}\|M(\xi,\theta)\|$.
\end{theorem}

\tmtextbf{Remark.} The settings of Theorem \ref{thm-1} are similar to those in \citet{li2021tss} for resolvable semantic transformations. But here we adopt a different presentation and proof for this theorem, which could be easier to extend to our GSmooth framework for general semantic transformations. 
Specifically, we show two examples of the theorem: additive transformations and commutable transformations. A transformation is additive if $\tau (\theta, \tau (\xi, x)) = \tau (\xi+\theta, x)$ for
any $\theta, \xi \in P$. It is commutable if $\tau (\theta, \tau (\xi, x)) = \tau (\xi, \tau
(\theta, x))$ for any $\theta, \xi \in P$. For these two types of transformations, it is straightforward to verify that they satisfy the property proposed in Theorem \ref{thm-1} with $M(\theta,\xi) = I$. Consequently, we simply apply Theorem \ref{thm-1} for an isotropic Gaussian distribution $g(\theta) = \mathcal{N}(0, \sigma^2 I)$ and
get the certified radius as
\begin{equation}
  R = \frac{\sigma}{2} \left( \Psi ( \underline{p_A} ) -
  \Psi ( \overline{p_B} ) \right),
\end{equation}
where $\Psi$ is the inverse CDF of the standard Gaussian distribution. These two kinds of transformations include image translation and Gaussian blur, which are basic semantic transformations and the results are consistent with previous work \citep{li2021tss,fischer2020certified}. The certification of these simple transformations only requires applying translation or Gaussian blur to the sample and we obtain the
average classification score under the noise distribution.

\subsection{Certified Bound for General Transformations}

Translation and Gaussian blur are two specific cases of
semantic transformations. In practice, most semantic transformations are not commutable or even not resolvable. Therefore, we need to develop better methods for certifying more types of semantic transformations. However, the existing methods like Semanify-NN~\citep{mohapatra2020towards}, based on convex relaxation, and TSS~\citep{li2021tss}, based on randomized smoothing, require developing a specific algorithm or bound for each individual semantic transformation. This is not scalable and might be infeasible for more complicated transformations without explicit mathematical forms. 

To address this challenge, we draw inspiration from the fact that neural networks are able to approximate functions including a complex and unknown semantic transformation \citep{zhu2017unpaired}. Specifically, we propose to use a surrogate image-to-image translation network to accurately fit the semantic transformation. We theoretically prove that, by introducing an augmented noise in the layer of the surrogate model, as shown in Fig.~\ref{fig1}, randomized smoothing can be extended to handle these transformations. 
Specifically, we define the surrogate model as the following form that will lead to a simple certified bound, as we shall see:
\begin{equation}
  \tau (\theta, x) = H (F_1 (\theta) + F_2 (x)),
\end{equation} 
where $F_1 (\cdummy) , F_2 (\cdummy) $, and $H (\cdummy)$ are three individual neural networks. $F_1(\cdot)$ and $F_2(\cdot)$ are encoders for transformation parameters and images respectively, and their encodings are added together in the semantic space, which is critical for the theoretical certification. We find that the surrogate neural network is much easier to analyze and can be certified by introducing a dimension augmentation strategy for both noise parameters and input images. 

As illustrated in Fig.~\ref{fig1}, the augmented noise is added to the semantic layer $H(\cdot)$ in the surrogate model.
Our key insight is that the transformation can be viewed as the superposition of a resolvable part and a non-resolvable residual part in the augmented semantic space. Then, we can use the augmented noise to control the non-resolvable residual part of the augmented dimension  $d\geqslant m+n$. This dimension augmentation is the key step in our method. The augmentation for noise is from $\mathbb{R}^m$ to $\mathbb{R}^d$. To keep the dimension consistent, we also augment data $x$ to $\mathbb{R}^d$ by padding 0 entries. 
By certifying the transformation using the surrogate model, we are able to certify the original transformation if the approximation error is within an acceptable region (details are in Theorem \ref{thm-3}). Our method is flexible and scalable because the surrogate neural network has a uniform mathematical form for theoretical analysis and is trained automatically. Next, we discuss the details of GSmooth. 

Specifically, 
we introduce the augmented data $\tilde{x} \in \mathbb{R}^d$ and the augmented parameter $\tilde{\theta} \in \mathbb{R}^d$ as
\begin{equation}
  \tilde{x} = \left( \begin{array}{c}
    \boldsymbol{0}_{d-n}\\
    x
  \end{array} \right),\ \tilde{\theta} = \left( \begin{array}{c}
    \theta\\
    \theta'
  \end{array} \right),
\end{equation}
where the additional parameters $\theta' \in \mathbb{R}^n$ are sampled from $g'
(\theta')$, and the joint distribution of $\theta'$ and $\theta$ is
$\tilde{\theta} \sim \tilde{g}$ where $\tilde{g} (\tilde{\theta}) = g'(\theta')g(\theta)$. 
We define the \emph{generalized
smoothed classifier} as
\begin{equation}
  \tilde{G} (\tilde{x}) =\mathbb{E}_{\tilde{\theta} \sim \tilde{g}
  (\cdot)} \left[\tilde{f} (\tilde{\tau} (\tilde{\theta}, \tilde{x}))\right],
  \label{2} 
\end{equation}
where $\tilde{f}$ is the ``augmented target classifier'' that is equivalent to the original classifier when constrained on the original input $x$, which means $\tilde{f}(\tilde{x}) = f(x)$.  Note that now all the functions are augmented for a $d$-dimensional input. We further augment our surrogate neural network to represent the augmented transformation
$\tilde{\tau}:\mathbb{R}^{d}\rightarrow \mathbb{R}^{d}$ as
\begin{equation}
  \tilde{\tau} (\tilde{\theta}, \tilde{x}) = \tilde{H} (\tilde{F_1}
  (\tilde{\theta}) + \tilde{F_2} (\tilde{x})),
\end{equation} 
where $\tilde{H}(\cdot), \tilde{F_1}(\cdot), \tilde{F_2}(\cdot): \mathbb{R}^d \to \mathbb{R}^d$ are parts of the augmented surrogate model. 
By carefully designing the interaction between the augmented parameters and
the original parameters, we could turn the transformation to a resolvable one. It does not change the original surrogate model when constraining to the original input $x$ and $\theta$.
Specifically, we design $\tilde{F_1}$, $\tilde{F_2}$ and $\tilde{H}$ as follows:
\begin{equation}
\begin{split}
  &\tilde{F_1} (\tilde{\theta}) = \left( \begin{array}{c}
    \theta\\
    F_{1} (\theta) + \theta'
  \end{array} \right), \tilde{F_2} (\tilde{x}) = \left( \begin{array}{c}
    \boldsymbol{0}_{d-n}\\
    F_{2} (x)
  \end{array} \right), \\
  &\tilde{H}(\tilde{x}) = 
  \begin{bmatrix}
  I_{d-n} & \\
  & H(x)  
  \end{bmatrix}.
\end{split}
\end{equation} 
Before stating our main theorem, to simplify the notations, we set
\begin{equation}
\begin{split}
    \tilde{z}_{\xi} &= \tilde{F}_1 (\tilde{\xi}) + \tilde{F}_2 (\tilde{x}),\quad \tilde{z}_{\theta} = \tilde{F_1} (\tilde{\theta}) + \tilde{F_2}(\tilde{x}),\\
    \tilde{y}_{\xi} &= (y_{\xi}', y_{\xi})^\top = \tilde{H} (\tilde{F_1} (\tilde{\xi}) + \tilde{F}_2 (\tilde{x})),\\ \tilde{y}_{\theta} &= (y_{\theta}', y_{\theta})^\top = \tilde{H}(\tilde{F_1} (\tilde{\theta}) + \tilde{F}_2 (\tilde{x})).
\end{split}
\end{equation}
Then we provide our main theoretical result below, i.e., the GSmooth classifier is certifiably robust within a given range.

\begin{theorem}[\textbf{Certified bound for Non-resolvable Transformations}, proof in Appendix \ref{app_thm_2}]
\label{thm-2}
For any classifier $f(x)$, let $\tilde{G} (\tilde{x})$ be the smoothed classifier defined in Eq.~(\ref{2}). If there exist $\underline{p_A}$ and $\overline{p_B}$ satisfying
  \[ \tilde{G} (\tilde{x})_A \geqslant \underline{p_A} \geqslant
     \overline{p_B} \geqslant \tilde{G} (\tilde{x})_B, \]
  then $ \margmax_{i \in \mathcal{Y}} \tilde{G} (\tilde{\tau}
  (\tilde{\xi}, \tilde{x}))_i = y_A$ holds for any $\| \xi \|_2 \leqslant R$, where
  \begin{equation}\label{eq:the2-r}
    R = \frac{1}{2M^{\ast}} \int_{\overline{p_B}}^{\underline{p_A}} \frac{1}{\Phi(p)}\mathd p,
  \end{equation}
  and the coefficient $M^*$ is defined as
  \begin{equation}
    M^{\ast} = \max_{\xi, \theta \in P} \sqrt{1 +
    \left\| \frac{\partial F_{2} (y_{\xi})}{\partial
    \xi} - \frac{\partial F_1 (\theta)}{\partial \theta} \right\|_2^2}.
    \label{3}
  \end{equation}
\end{theorem}
We have several observations from the theorem. First, it is noted that the certified radius is
similar to the result in Theorem \ref{thm-1}. Second, compared with resolvable
transformations, we need to introduce a new type of noise when constructing the
GSmooth classifier. This isotropic noise has the same dimension
as the data and is added to the intermediate layers of surrogate neural
networks. The theoretical explanation behind this is that this isotropic noise allows the Jacobian matrix of the semantic transformation to be invertible, which is crucial for the proof. Third, we find that the coefficient $M^{\ast}$ depends on the
norm of the difference of two Jacobian matrices and is independent of the target classifier. Later we will discuss the
meaning of this term in detail.

Before diving into our theoretical insight and proof of the theorem, we
introduce a specific case of Theorem \ref{thm-2} which is convenient for practical
usage. We empirically found that taking a linear transformation as $F_1
(\theta) = A_1 \theta + b_1$ where $A_1 \in \mathbb{R}^{n \times m}$, $b_1 \in
\mathbb{R}^n$ does not sacrifice the precision of the surrogate network. So, we have $\frac{\partial F_1 (\theta)}{\partial \theta} = A_1$. After
substituting the term in Eq.~(\ref{3}), we only need to optimize $\xi$ to
calculate $M^{\ast}$ which can make the bound tighter. Additionally, we use two Gaussian distributions as the noise distribution and the augmented noise distribution, i.e., $g(\theta)=\mathcal{N}(0,\sigma_1^2I)$ and $g'(\theta')=\mathcal{N}(0,\sigma_2^2 I)$. Formally, we have the following corollary.
\begin{corollary}
\label{cor-1}
Suppose $f (x)$ is a classifier and $\tilde{G} (\tilde{x})$ is the smoothed
  classifier defined in Eq.~(\ref{2}). If the layer $F_1 (\theta)$ in the
  surrogate neural network has the following form:
  \begin{equation}
    F_1 (\theta) = A_1 \theta + b_1,
  \end{equation}
  where $A_1 {\in \mathbb{R}^{d \times m}} $, $b_1 \in \mathbb{R}^d$ are the
  parameters; and if there exist $\underline{p_A}$ and $\overline{p_B}$ satisfying
  \[ \tilde{G} (\tilde{x})_A \geqslant \underline{p_A} \geqslant
     \overline{p_B} \geqslant \tilde{G} (\tilde{x})_B, \]
  then $ \margmax_{i \in \mathcal{Y}} \tilde{G} (\tilde{\tau}
  (\tilde{\xi}, \tilde{x}))_i = y_A$ for any $\| \xi \|_2 \leqslant R$ where
  \begin{equation}
    R = \frac{1}{2 M^{\ast}} \left( \Psi \left( \underline{p_A} \right)
    - \Psi \left( \overline{p_B} \right) \right),
  \end{equation}
  where $\Psi(\cdot)$ is the inverse CDF of the standard Gaussian distribution, and the coefficient $M^*$ is defined as
  \begin{equation}
    M^{\ast} = \max_{\xi \in P} \sqrt{\frac{1}{\sigma_1^2} + \frac{1}{\sigma_2^2}
    \left\| \frac{\partial F_{2} (y_{\xi})}{\partial \xi} - A_1 \right\|_2^2}.
    \label{4}
  \end{equation}
\end{corollary}

\section{Proof Sketch and Theoretical Analysis}
In this section, we briefly summarize the main idea for proving Theorem~\ref{thm-2}
and provide the theoretical insights for our GSmooth further. More details of the proof
can be found in Appendix \ref{app_thm_2}.

\subsection{Proof Sketch of Theorem~\ref{thm-2}}
The key idea is to prove that the gradient of the smoothed classifier can be bounded by a function of the classification confidence and the parameters of the noise distribution. Formally, we calculate the gradient to the perturbation parameter $\xi$ for our augmented smoothed classifier as
\begin{equation}
  \nabla_{\tilde{\xi}} \tilde{G} (\tilde{\tau} (\tilde{\xi}, \tilde{x})) =
  \nabla_{\tilde{\xi}} \mathbb{E}_{\tilde{\theta} \sim \tilde{g}
  (\tilde{\theta})} [\tilde{f} (\tilde{\tau} (\tilde{\theta}, \tilde{\tau}
  (\tilde{\xi}, \tilde{x})))] .
\end{equation}
We further expand the expectation into an integral and find that
\begin{equation}
   \nabla_{\tilde{\xi}} \tilde{G} (\tilde{\tau} (\tilde{\xi}, \tilde{x})) =
  \int_{\mathbb{R}^{ d}} \frac{\partial \tilde{f} (\tilde{\tau}
  (\tilde{\theta}, \tilde{y}_{\xi}))}{\partial 
  \tilde{\xi}}  \tilde{g} (\tilde{\theta}) \mathd
  \tilde{\theta} .
\end{equation}
The key step is to eliminate the gradient of $\frac{\partial \tilde{F}
(\tilde{\tau} (\tilde{\theta}, \tilde{y}_{\xi}))}{\partial \tilde{\xi}}$ and replace it with $\frac{\partial
\tilde{f} (\tilde{\tau} (\tilde{\theta}, \tilde{y}_{\xi}))}{\partial
\tilde{\theta}}$. Then we integrate it by parts to get the following
objective:
\begin{equation}
  \nabla_{\tilde{\xi}} \tilde{G} (\tilde{\tau} (\tilde{\xi}, \tilde{x})) = -
  \int_{\mathbb{R}^{d}} \tilde{F} (\tilde{\tau} (\tilde{\theta},
  \tilde{y}_{\xi})) \frac{\partial \big(\tilde{M}
  (\tilde{\xi}, \tilde{\theta}) \tilde{g} (\tilde{\theta})\big)}{\partial \tilde{\theta}}  \mathd
  \tilde{\theta} .
\end{equation}
After that, we can bound the gradient of the GSmooth classifier by
using a technique similar to randomized smoothing \citep{yang2020randomized}. 

\subsection{Theoretical Insights}
Next, we provide some theoretical insights for our augmentation scheme on
transformation parameters.

\begin{itemize}
    \item \textbf{Transformation space expansion.} The key purpose of the augmented noise is to construct a closed subspace using additional dimensions. In the augmented space, the Jacobian matrix of the semantic transformation becomes invertible, which is crucial for our proof.
    
    \item \textbf{Decomposition of the non-resolvable transformations.} As we can see in Eq.~(\ref{4}),
    $M^{\ast}$ is influenced by two factors. One is the standard deviation of the
    two noise distributions. The other is the norm of the Jacobian matrix
    $\frac{\partial F_{2} (y_{\xi})}{\partial \xi} - A_1$. It can be
    viewed as the residual of the non-resolvable part of the transformation. Along
    these lines, our method decomposes the unknown semantic transformation into a
    resolvable part and a residual part. The non-resolvable residual part can be
    handled by introducing an additional noise with standard deviation $\sigma_2$.
    
\end{itemize}

\subsection{Error Analysis for Surrogate Model Approximation}
In this subsection, we theoretically analyze the effectiveness of certifying a real semantic transformation due to the existence of approximation error of surrogate neural networks. 
\begin{theorem}[\textbf{Error Analysis}, proof in Appendix~\ref{app_thm_3}]
\label{thm-3}
  Suppose the simulation of the semantic transformation has a small enough error
  \[ \| \tilde{\tau} (\tilde{\xi}, \tilde{x}) - \overline{\tau} (\tilde{\xi}, \tilde{x})
     \|_2 < \varepsilon, \]
    where $\overline{\tau} (\tilde{\xi}, \tilde{x})$ is the real semantic transformation.
  Then there exists a constant ratio 
  \[A=A(\|F_1'(\tilde{\xi})\|_2, \|F_2'(\tilde{y}_\xi)\|_2,\|F_2'(\tilde{z}_\xi)\|_2)>0,\]
  which does not depend on the target classifier. The certified radius for the real semantic transformation satisfies: 
  \begin{equation}
      R_r > R (1 - A \varepsilon),
  \end{equation} 
  where $R$ is the certified radius in Eq.~\eqref{eq:the2-r} for the surrogate neural network in Theorem \ref{thm-2}.
\end{theorem}
We find that the reduction of the certified radius is influenced by two factors. The first one is the approximation error $\epsilon$ between the surrogate transformation and the real semantic transformation. The second one, the ratio $A$, is about the norm of the Jacobian matrix for some layers of the surrogate model. This is also an inherent property of the semantic transformation itself and does not depend on the target classifier.

\section{Experiments}
\label{5}
In this section, we conduct extensive experiments to show the effectiveness of our GSmooth on various types of semantic transformations. 

\subsection{Experimental Setup and Evaluation Metrics}
\label{5-1}

\begin{table*}[!t]
\centering
\footnotesize
\caption{Our main results of certified robust accuracy on several datasets and multiple types of semantic transformations. -- or $0.0$\% means the method fails to certify this type of semantic transformation.}

\setlength{\tabcolsep}{3pt}
\begin{tabular}{l|ccl|ccccccc}
\hline
\multirow{3}{*}{Transformation} & \multirow{3}{*}{Type} & \multirow{3}{*}{Dataset} & \multirow{3}{*}{ Certified Radius} & \multicolumn{7}{c}{Certified Accuracy ($\%$)} \\
 &  &  &  & \textbf{GSmooth} & TSS & DeepG & Interval & VeriVis & Semanify- & IndivSPT/ \\
 &  &  &  & \textbf{(Ours)} &  &  &  &  & NN & distSPT \\ \hline
\multirow{3}{*}{Gaussian Blur} & \multirow{3}{*}{Additive} & MNIST & $\|\alpha\|_2<6$ & \textbf{91.0} & 90.6 & -- & -- & -- & -- & -- \\
 &  & CIFAR-10 & $\|\alpha\|_2<4$ & \textbf{67.4} & 63.6 & -- & -- & -- & -- & -- \\
 &  & CIFAR-100 & $\|\alpha\|_2<4$ & \textbf{22.1} & 21.0 & -- & -- & -- & -- & -- \\ \hline
\multirow{3}{*}{Translation} & \multirow{3}{*}{Additive} & MNIST & $\|\alpha\|_2<8$ & 98.7 & \textbf{99.6} & 0.0 & 0.0 & 98.8 & 98.8 & \textbf{99.6} \\
 &  & CIFAR-10 & $\|\alpha\|_2<20$ & \textbf{82.2} & 80.8 & 0.0 & 0.0 & 65.0 & 65.0 & 78.8 \\
 &  & CIFAR-100 & $\|\alpha\|_2<20$ & \textbf{42.2} & 41.3 & -- & -- & 24.2 & 24.2 & -- \\ \hline
\multirow{3}{*}{\shortstack[l]{Brightness\\Contrast}} & \multirow{3}{*}{Resolvable} & MNIST & $\|\alpha\|_{\infty}<0.5$ & \textbf{97.7} & 97.6 & $\leqslant$0.4 & 0.0 & -- & $\leqslant$74 & -- \\
 &  & CIFAR-10 & $\|\alpha\|_{\infty}<0.4$ & \textbf{82.5} & 82.4 & 0.0 & 0.0 & -- & -- & -- \\
 &  & CIFAR-100 & $\|\alpha\|_{\infty}<0.4$ & \textbf{42.3} & 41.4 & 0.0 & 0.0 & -- & -- & -- \\ \hline
\multirow{3}{*}{Rotation} & \multirow{3}{*}{Non-resolvable} & MNIST & $\|\alpha\|_2<50^\circ$ & 95.7 & \textbf{97.4} & $\leqslant$85.8 & $\leqslant$6.0 & -- & $\leqslant$92.48 & $\leqslant$76 \\
 &  & CIFAR-10 & $\|\alpha\|_2<10^\circ$ & 65.6 & \textbf{70.6} & 62.5 & 20.2 & -- & $\leqslant$49.37 & $\leqslant$34 \\
 &  & CIFAR-100 & $\|\alpha\|_2<10^\circ$ & 33.2 & \textbf{36.7} & 0.0 & 0.0 & -- & $\leqslant$21.7 & $\leqslant$18 \\ \hline
\multirow{3}{*}{Scaling} & \multirow{3}{*}{Non-resolvable} & MNIST & $\|\alpha\|_2 < 0.3$ & 95.9 & \textbf{97.2} & 85.0 & 16.4 & -- & -- & -- \\
 &  & CIFAR-10 & $\|\alpha\|_2<0.3$ & 54.3 & \textbf{58.8} & 0.0 & 0.0 & -- & -- & -- \\
 &  & CIFAR-100 & $\|\alpha\|_2<0.3$ & 31.2 & \textbf{37.8} & 0.0 & 0.0 & -- & -- & -- \\ \hline
\multirow{3}{*}{Rotational Blur} & \multirow{3}{*}{Non-resolvable} & MNIST & $\|\alpha\|_2<10$ & \textbf{95.9} & -- & -- & -- & -- & -- & -- \\
 &  & CIFAR-10 & $\|\alpha\|_2<10$ & \textbf{39.7} & -- & -- & -- & -- & -- & -- \\
 &  & CIFAR-100 & $\|\alpha\|_2<10$ & \textbf{27.2} & -- & -- & -- & -- & -- & -- \\ \hline
\multirow{3}{*}{Defocus Blur} & \multirow{3}{*}{Non-resolvable} & MNIST & $\|\alpha\|_2<5$ & \textbf{89.2} & -- & -- & -- & -- & -- & -- \\
 &  & CIFAR-10 & $\|\alpha\|_2<5$ & \textbf{25.0} & -- & -- & -- & -- & -- & -- \\
 &  & CIFAR-100 & $\|\alpha\|_2<5$ & \textbf{13.1} & -- & -- & -- & -- & -- & -- \\ \hline
\multirow{3}{*}{Zoom Blur} & \multirow{3}{*}{Non-resolvable} & MNIST & $\|\alpha\|_2<0.5$ & \textbf{93.9} & -- & -- & -- & -- & -- & -- \\
 &  & CIFAR-10 & $\|\alpha\|_2<0.5$ & \textbf{44.6} & -- & -- & -- & -- & -- & -- \\
 &  & CIFAR-100 & $\|\alpha\|_2<0.5$ & \textbf{14.2} & -- & -- & -- & -- & -- & -- \\ \hline
\multirow{3}{*}{Pixelate} & \multirow{3}{*}{Non-resolvable} & MNIST & $\|\alpha\|_2<0.5$ & \textbf{87.1} & -- & -- & -- & -- & -- & -- \\
 &  & CIFAR-10 & $\|\alpha\|_2<0.5$ & \textbf{45.3} & -- & -- & -- & -- & -- & -- \\
 &  & CIFAR-100 & $\|\alpha\|_2<0.5$ & \textbf{30.2} & -- & -- & -- & -- & -- & -- \\ \hline
\end{tabular}
\label{tab-1}
\end{table*}

\begin{table}[!t]
\footnotesize
\caption{Results of empirical accuracy under adaptive attack on CIFAR-10 (PGD using EoT). For comparison, we also list the results of the corresponding certified accuracy in Table \ref{tab-1}.}

\setlength{\tabcolsep}{2.8pt}
\begin{tabular}{l|cp{8ex}<{\centering}c}
\hline
\multirow{2}{*}{Type} & \multicolumn{1}{c|}{Certified Acc. (\%)} & \multicolumn{2}{c}{Adaptive Attack Acc. (\%)} \\ \cline{2-4} 
                      & \multicolumn{2}{c|}{GSmooth}                            & Vanilla            \\ \hline
Gaussian blur         & 67.4                             & 68.1             & 3.4              \\
Translation           & 82.2                             & 87.5             & 4.2              \\
Brightness            & 82.5                             & 85.9             & 9.6              \\
Rotation              & 65.6                             & 68.4             & 65.4             \\
Rotational blur       & 39.7                             & 45.0             & 33.1             \\
Defocus blur          & 25.0                             & 25.5             & 16.6             \\
Pixelate              & 45.3                             & 49.2             & 38.2             \\ \hline
\end{tabular}
\label{ada_result}
\end{table}

\begin{table*}[th]
\centering
\caption{Results of ablation study on the influence of standard deviation of smoothing distributions, i.e., transformation noise $\sigma_1$ and augmented noise $\sigma_2$ for certification accuracy on rotational blur.}
\footnotesize
\begin{tabular}{c|c|cccc|ccccc}
\hline
Cert Acc. (\%) & \multicolumn{5}{c|}{CIFAR-10}              & \multicolumn{5}{c}{CIFAR-100}                                   \\ \hline
\multirow{5}{*}{Rotational Blur} &
  \diagbox{$\sigma_1$}{$\sigma_2$} &
  0.05 &
  0.10 &
  0.15 &
  0.25 &
  \multicolumn{1}{c|}{\diagbox{$\sigma_1$}{$\sigma_2$}} &
  0.05 &
  0.10 &
  0.15 &
  0.20 \\ \cline{2-11} 
         & 0.1  & 44.3 & 46.4          & 35.5 & 16.1 & \multicolumn{1}{c|}{0.1}  & 23.1 & 26.2          & 17.2 & 10.4 \\
         & 0.25 & 45.7 & 47.1          & 38.3 & 18.9 & \multicolumn{1}{c|}{0.25} & 23.2 & \textbf{27.2} & 20.3 & 11.3 \\
         & 0.5  & 46.5 & \textbf{48.4} & 38.8 & 18.3 & \multicolumn{1}{c|}{0.5}  & 22.2 & 26.1          & 18.6 & 11.8 \\
         & 0.75 & 42.1 & 45.5          & 36.6 & 17.0 & \multicolumn{1}{c|}{0.75} & 24.0 & 25.5          & 18.3 & 13.2 \\ \hline
\end{tabular}
\label{tab-3}
\end{table*}

We use MNIST \citep{726791}, CIFAR-10, and CIFAR-100 \citep{krizhevsky2009learning} datasets to verify our methods. We train a ResNet with 110 layers \citep{he2016deep} from scratch. Similar to prior work, we apply moderate data augmentation \citep{cohen2019certified} to improve the generalization of the classifier. For the surrogate image-to-image translation model for simulating semantic transformations, we adopt the U-Net architecture \citep{ronneberger2015u} for $\tilde{H}(\cdot)$ and several simple convolutional or linear layers for $\tilde{F_1}(\cdot)$ and $\tilde{F_2}(\cdot)$. All models are trained by using an Adam optimizer~\cite{adam} with an initial learning rate of 0.001 that decays every 50 epochs until convergence. The algorithms for calculating $M^\ast$ and other details in our experiments are listed in Appendix \ref{appendix2} due to limited space. The evaluation metric is the certified accuracy, which is the percentage of samples that are correctly classified with a larger certified radius than the given range. We use $\|\alpha\|$ to indicate the preset certified radius. The results are evaluated on the real semantic transformations after considering the error correction in Theorem \ref{thm-3}.

\subsection{Main Results for Certified Robustness}
\label{5-2}

To demonstrate the effectiveness of our GSmooth on certifying complex semantic transformations, we compare the results of our GSmooth with several baselines, including \emph{randomized smoothing} for some simple semantic transformations of  \textbf{TSS}~\citep{li2021tss} and \textbf{IndivSPT/distSPT}~\citep{fischer2020certified}. Our GSmooth is a natural and powerful extension of their methods. We also compare our method with \emph{deterministic certification} approaches, including \textbf{DeepG}~\citep{balunovic2019certifying}, which uses linear relaxations similar to~\citet{wong2018provable},  \textbf{Interval}~\citep{singh2019abstract}, which is based on interval bound propagation,  \textbf{VeriVis}~\citep{pei2017towards}, which enumerates all possible outcomes for semantic transformations with finite values of parameters, and \textbf{Semanify-NN}~\citep{mohapatra2020towards}, which uses a new preprocessing layer to turn the problem into a $\ell_p$ norm certification.

We measure the certified robust accuracy for different semantic transformations on different datasets in Table \ref{tab-1}. We have the following observations based on the experimental results. First, only our method achieves non-zero accuracy in certifying some complex semantic transformations, including rotational blur, defocus blur, zoom blur, and pixelate. The results verify Theorem \ref{thm-2}. This is a breakthrough that greatly extends the range of applicability for methods based on randomized smoothing. Second, we see that the performance of GSmooth is similar to state-of-the-art randomized smoothing approaches (e.g., TSS) on several simple semantic transformations such as Gaussian blur and translation. This is a natural result because our method works similarly for resolvable transformations. For two specific non-resolvable transformations, i.e., rotation and scaling, our accuracy is slightly lower. The possible reason is that, in TSS \citep{li2021tss}, they derive a more elaborate Lipschitz bound for rotation, which is better than us. Third, those inherently non-resolvable transformations such as image blurring (except Gaussian blur) and pixelate are more difficult than resolvable or approximately resolvable (rotation) transformations, which makes their certified accuracy lower than the resolvable ones.


\subsection{Results for Empirical Robustness}
To demonstrate that the certified radius is tight enough for real applications, we conduct empirical robustness experiments by testing our method under two types of corruptions. First, we apply project gradient descent (PGD)~\cite{madry2018towards} using Expectation over Transformation (EoT)~\citep{Athalye2018Obfuscated} and evaluate the accuracy under adversarial examples on CIFAR-10. Second, we measure the empirical performance of GSmooth under common image corruptions such as CIFAR-10-C \citep{hendrycks2019benchmarking}. We choose AugMix \citep{hendrycks2019augmix} and TSS \citep{li2021tss} as baselines. For evaluation, we choose subsets of CIFAR-10-C that are attacked by deterministic corruptions.

\begin{table}[!t]
\caption{Empirical accuracy (\%) on subsets of CIFAR-10-C. Results with the best performance are bolded. }
\footnotesize
\centering
\begin{tabular}{l|ccc}
\hline
Type & AugMix          & TSS    & GSmooth   \\ \hline
Gaussian blur      & 67.4          & 75.8 & \textbf{76.0} \\
Brightness         & \textbf{82.4} & 71.8 & 72.1          \\
Defocus blur       & 72.2          & 75.6 & \textbf{76.8} \\
Zoom blur          & 70.8          & 75.2 & \textbf{77.1} \\
Motion blur        & 68.6          & 70.2 & \textbf{70.5} \\
Pixelate           & 50.9          & 76.0 & \textbf{76.7} \\ \hline
\end{tabular}
\label{cifar-c_result}
\end{table}

The performance under adaptive attack is shown in Table~\ref{ada_result}. From the results, we find that the empirical accuracy is always higher than the certified accuracy, which indicates the tightness of the certified bound for GSmooth. Additionally, GSmooth also serves as a strong empirical defense compared with a vanilla neural network. The results on CIFAR-10-C are listed in Table \ref{cifar-c_result}. We see that our method achieves improvement on most types of corruptions. Since the corruptions in CIFAR-10-C are randomly produced and model-agnostic, the overall performance is better than adaptive attacks. In summary, these empirical results show that the certified bound is tight enough and that GSmooth can also be used as a strong empirical defense method.


\begin{figure}
    \centering
    \hbox{\includegraphics[width=0.99\linewidth,page=1]{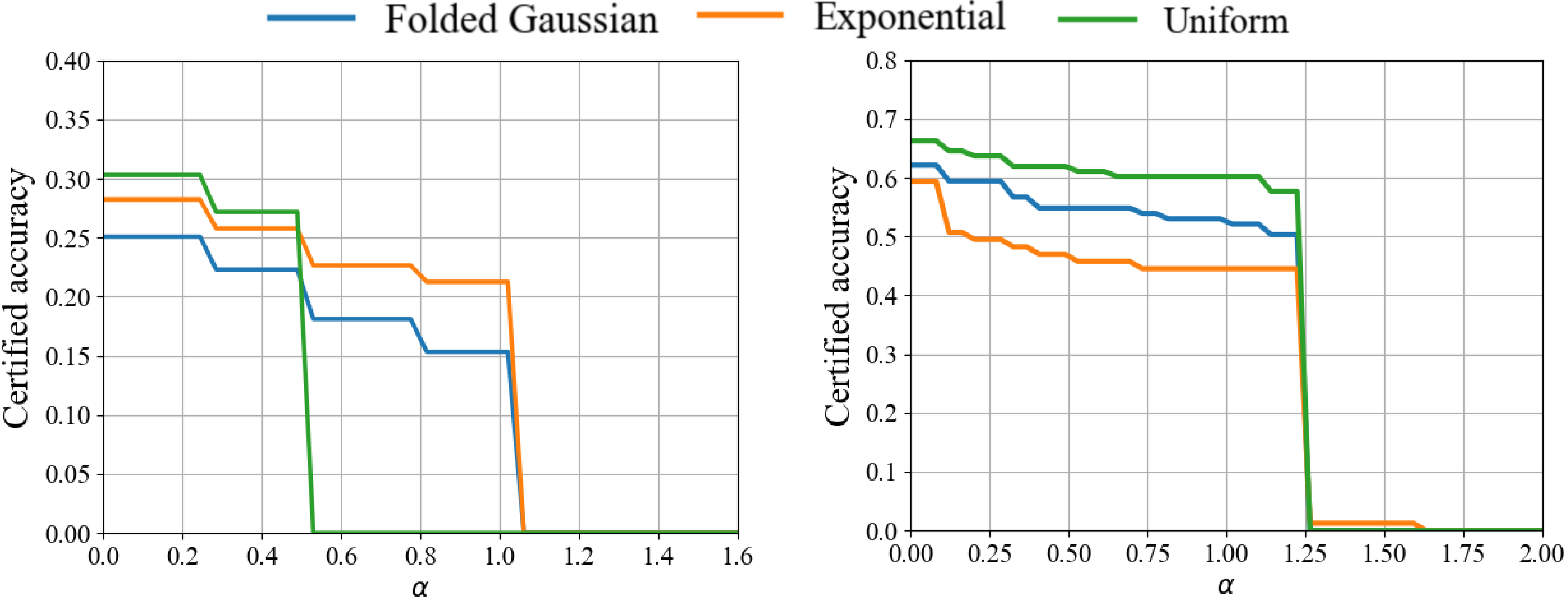}}
    \caption{Results of ablation experiments on the influence of smoothing distribution for zoomed blur on the CIFAR-10 dataset. The horizontal axis $\|\alpha\|_2$ is the certified radius. }
    \label{ab1}
\end{figure}


\subsection{Ablation Studies}
\label{5-3}
\textbf{Ablation study on the influence of noise distribution.}
The choice of noise distribution is important for randomized smoothing based methods. Since our GSmooth can certify different types of semantic transformations that exhibit different properties, it is necessary to 
understand the influence of different smoothing distributions for different semantic transformations. We choose (folded) Gaussian, uniform, and exponential distribution and compare the certified accuracy on zoom blur transformation for both CIFAR-10 and CIFAR-100 datasets. As shown in Fig.~\ref{ab1}, we found that the impact of smoothing distributions depends on the dataset. On average, uniform distribution is better for small radius certification, while exponential distribution is more suitable for certifying a large radius.

\begin{figure}[!t]
    \centering
    \includegraphics[width=8cm]{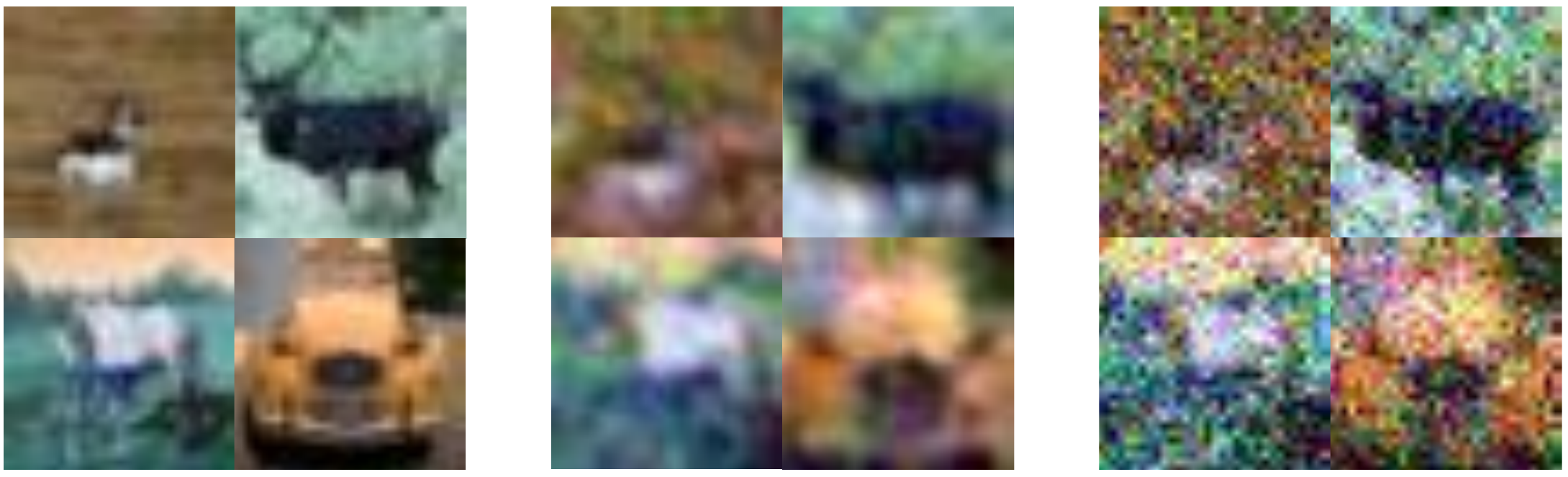}
    \caption{Visualization of difference between the augmented noise in the semantic layers and the noise on raw images. Left: original images from CIFAR-10. Middle: images with augmented noise of $\sigma_2=0.2$. Right: images with additive Gaussian noise $\sigma=0.2$.}
    \label{fig3}
\end{figure}
\textbf{Ablation study on the influence of noise variance for certification.}
Since our GSmooth contains two different variances for controlling the resolvable part and the residual part for a non-resolvable semantic transformation, here we investigate the effect of different noise variance on the certified accuracy. The results are shown in Table \ref{tab-3}. We found that using medium transformation noise and augmented noise achieves the best certified accuracy. This fact is consistent with results in \citet{cohen2019certified}. An explanation is that there is a trade-off: higher noise variance decreases the coefficient $M^\ast$, but it might also degrade the clean accuracy.

\textbf{Visualization experiments: comparison between augmented semantic noise and image noise.}
Our GSmooth needs to add a new noise in the semantic layers of the surrogate model. Here we compare the difference between these two types of noise and visualize them. We randomly sample images from CIFAR-10 and add Gaussian noise with $\sigma=0.2$ to both the semantic layer of the surrogate model simulating zoomed blur transformation and raw images. Results are shown in Fig.~\ref{fig3}. Both types of noise severely blur the images. But we found that the augmented semantic noise is more placid, which can therefore keep the holistic semantic features better, e.g., as shapes.  
\label{others}
\section{Conclusion}
In this paper, we proposed a generalized randomized smoothing framework (GSmooth) for certifying robustness against general semantic transformations. Considering that it is nontrivial to conduct theoretical analysis for non-resolvable transformations, we proposed a novel method using a surrogate neural network to fit semantic transformations. Then we theoretically provide a tight certified robustness bound for the surrogate model, which can be used for certifying semantic transformations. Extensive experiments on publicly available datasets verify the effectiveness of our method by achieving state-of-the-art performance on various types of semantic transformations especially for those complex transformations such as zoom blur or pixelate. 

\section*{Acknowledgments}

This work was supported by the National Key Research and Development Program of China (2020AAA0106000, 2020AAA0104304, 2020AAA0106302, 2017YFA0700904), NSFC Projects (Nos. 62061136001, 61621136008, 62076147, U19B2034, U19A2081, U1811461), the major key project of PCL (No. PCL2021A12), 
Tsinghua-Alibaba Joint Research Program, Tsinghua-OPPO Joint Research Center, and Tiangong Institute for Intelligent Computing, NVIDIA NVAIL Program and the High Performance Computing Center, Tsinghua University.




\bibliographystyle{icml2022}
\bibliography{ref}
\newpage
\appendix
\onecolumn

\section{Proofs of Theorems}
\label{appendix1}
In this section, we will provide detailed proofs of theorems in our paper.

First, we restate the theorem of randomized smoothing for additive noise and
binary classifiers $f (\cdummy) : \mathbb{R}^n \rightarrow [0, 1]$,
\begin{theorem}
\label{thm-4}
  Let $f (x)$ be any classifier and $G (x)$ be the smoothed classifier defined as
  \begin{equation}
  G (x) =\mathbb{E}_{\theta \sim g } [f (x + \theta)], \label{a1}
\end{equation}
  if $G (x) < \frac{1}{2}$, then $G (x + \delta) <
  \frac{1}{2}$ holds for any
  \begin{equation}
    \| \delta \| < \int_{G (x)}^{\frac{1}{2}} \frac{1}{\Phi (p)} \mathd p,
  \end{equation}
  where $\Phi (\cdummy)$ is a function about smoothing distribution defined in
  Eq.~(\ref{eq2}).
\end{theorem}

\begin{proof}
  We first calculate the gradient of the smoothed classifier as
  \begin{eqnarray*}
    \nabla G (x) & = & \frac{\partial}{\partial x} \int_{\mathbb{R}^n} f (x +
    \theta) g (\theta) \mathd \theta\\
    & = & \int_{\mathbb{R}^n} \frac{\partial}{\partial x} f (x + \theta) g(\theta) \mathd \theta\\
    & = & \int_{\mathbb{R}^n} \left(\frac{\partial}{\partial \theta} f (x + \theta)\right) g(\theta) \mathd \theta.
  \end{eqnarray*}
  Then we multiple any vector with unit norm $u \in \mathcal{B}_n (1) =
  \left\{ u : \| u \| {= 1, u \in \mathbb{R}^n}  \right\}$,
  \begin{eqnarray}
    | \langle \nabla G (x), u \rangle | & = & \left| \left\langle
    \int_{\mathbb{R}^n} \left(\frac{\partial}{\partial \theta} f (x + \theta)\right) g(\theta) \mathd \theta, u \right\rangle \right| \nonumber\\
    & = & \left| \int_{\mathbb{R}^n} \left\langle \frac{\partial}{\partial
    \theta} f (x + \theta), u \right\rangle g (\theta) \mathd \theta \right|
    \nonumber\\
    & = & \left| \sum_i \int_{\mathbb{R}^n} \frac{\partial f (x +
    \theta)}{\partial \theta_i} u_i g (\theta) \mathd \theta \right|
    \nonumber\\
    & = & \left| \sum_i \int_{\mathbb{R}^{n - 1}} \left( \int_{- \infty}^{+
    \infty} \frac{\partial f (x + \theta)}{\partial \theta_i} u_i g (\theta)
    \mathd \theta_i \right) \prod_{j \neq i} \mathd \theta_j \right|
    \nonumber\\
    & \overset{1}{=} & \left| - \sum_i \int_{\mathbb{R}^{n - 1}} \left( \int_{- \infty}^{+
    \infty} f (x + \theta) u_i \frac{\partial g (\theta)}{\partial \theta_i}
    \mathd \theta_i \right) \prod_{j \neq i} \mathd \theta_j \right|
    \nonumber\\
    & = & \left| \sum_i \int_{\mathbb{R}^n} f (x + \theta) \frac{\partial g
    (\theta)}{\partial \theta_i} u_i \mathd \theta \right| \nonumber\\
    & = & \left| \int_{\mathbb{R}^n} f (x + \theta) \langle \nabla g
    (\theta), u \rangle \mathd \theta \right| \nonumber\\
    & = & \left| \int_{\mathbb{R}^n} f (x + \theta) g (\theta) \langle \nabla
    \psi (\theta), u \rangle \mathd \theta \right| \nonumber\\
    & = & | \mathbb{E}_{\theta \sim g } [f (x + \theta) \langle
    \nabla \psi (\theta), u \rangle] |, 
  \end{eqnarray}
  here equation 1 holds because of integration by parts. To bound the gradient of the smoothed classifier, we use the following
  inequality,
  \begin{equation}
    | \langle \nabla G (x), u \rangle | \leqslant \sup_{\widehat{f :} \hat{G}
    (x) = G (x)} \mathbb{E}_{\theta \sim g } \left[\hat{f} (x + \theta)
    \langle  \nabla\psi (\theta), u \rangle\right].
  \end{equation}
  As shown in \citet{yang2020randomized}, the optimal $\hat{f} (x)$ achieves at,
  \begin{equation}
    \hat{f} (x + \theta) = \left\{\begin{array}{l}
      1,\quad \tmop{if} \langle u, \nabla\psi (\theta) \rangle > \varphi^{- 1}_u (G (x))
      \\
      0,\quad \tmop{else}.
    \end{array}\right.
  \end{equation}
  Then, we have
  \begin{eqnarray}
    \mathbb{E}_{\theta \sim g } \left[\hat{f} (x + \theta) \langle u, \nabla\psi
    (\theta) \rangle\right] & = & \mathbb{E} \left[\gamma_u \mathbb{I} \{ \gamma_u >
    \varphi^{- 1}_u (G (x)) \}\right] \nonumber\\
    & \leqslant & \Phi (G (x)),
  \end{eqnarray}
  which means that,
  \begin{equation}
    | \langle \nabla G (x), u \rangle | \leqslant \Phi (G (x)).
  \end{equation}
  Since it holds for all $u \in \mathcal{B}_n (1)$, i.e., $\|u\|\leq = 1$, we have,
  \begin{equation}
    \max_{u \in \mathcal{B}_n (1)} \langle \nabla G (x), u \rangle \leqslant
    \Phi (G (x)).
  \end{equation}
  Consider a path from $\xi_t : [0, \| \delta \|] \rightarrow \mathbb{R}^d$
  with $\xi_0 = x$, $\xi_{\| \delta \|} = x + \delta$ and $\xi_t' =
  \frac{\delta}{\| \delta \|}$, we have
  \begin{equation}
    \frac{\mathd G (\xi_t)}{\mathd t} = \langle \nabla G (\xi_t), u \rangle
    \leqslant \Phi (G (\xi_t)) .
  \end{equation}
  If the norm of $\delta$ satisfies that,
  \begin{equation}
  \label{app-27}
    \| \delta \| < \int_{G (x)}^{\frac{1}{2}} \frac{1}{\Phi (p)} \mathd p.
  \end{equation}
  If the right hand side of eq (\ref{app-27}) exists, we name it $\| \delta_0 \| \triangleq \int_{G
  (x)}^{\frac{1}{2}} 1 / \Phi (p) \mathd p$. WLOG, we
  assume that $G (\xi_t)$ is increasing in $t$, then we could calculate the
  minimal $t$ when $G (\xi_t)$ increase to $\frac{1}{2}$,
  \begin{equation}
    T = \int_{G (\xi_0)}^{\frac{1}{2}} \frac{1}{\Phi (p)} \mathd p = \|
    \delta_0 \| .
  \end{equation}
  By the generality of $\delta$ we have $G (x + \delta) < \frac{1}{2}$ holds for any $\delta < \| \delta_0 \|$.
\end{proof}

This theorem can be naturally extended to problems with $p > 2$ classes by considering the two top classes which are $G(x)_A$ and $G(x)_B$. This turns the problem into a binary classification problem.

\subsection{The Proof of Theorem \ref{thm-1}}
\label{app_thm_1}
In this part, we will prove Theorem~\ref{thm-1}.
\begin{reptheorem}{thm-1}
For any classifier $f (x)$, let $G (x)$ be the smoothed classifier defined in Eq.~(\ref{1}). If there exists a function $M(\cdot,\cdot):P\times P\rightarrow \mathbb{R}$ satisfying
  \[ \frac{\partial \gamma (\theta, \xi)}{\partial \xi} = \frac{\partial
     \gamma (\theta, \xi)}{\partial \theta} M (\theta, \xi), \]
  and there exist two constants $\underline{p_A}$, $\overline{p_B}$ satisfying
  \[ G (x)_A \geqslant \underline{p_A} \geqslant \overline{p_B} \geqslant G
     (x)_B, \]
  then $y_A = \margmax_{i \in \mathcal{Y}} G (\tau (\xi, x))_i$ holds for any
  $\| \xi \| \leqslant R$, where
  \begin{equation}
    R = \frac{1}{ 2M^*} \int_{\overline{p_B}}^{\underline{p_A}} \frac{1}{\Phi(p)}\mathd p ,
  \end{equation}
and $M^* = \max_{\xi,\theta \in P}\|M(\xi,\theta)\|$.
\end{reptheorem}

\begin{proof}
WLOG, we only prove it for binary cases that $f(\cdot):\mathbb{R}^n\to[0,1]$. Recall that $G (x) =\mathbb{E}_{\theta \sim g} [f (\tau (\theta, x))]$, we have
\begin{equation}
\begin{split}
  \nabla_{\xi} G (\tau(\gamma (\theta, \xi), x))
  = & \int \frac{\partial f (\tau(\gamma (\theta, \xi), x))}{\partial \tau(\gamma (\theta, \xi), x)} 
  \frac{\partial \tau(\gamma (\theta, \xi), x)}{\partial \xi} g (\theta) \mathd
  \theta \label{34} \\
  = & \int \frac{\partial f (\tau(\gamma (\theta, \xi), x))}{\partial \tau(\gamma (\theta, \xi), x)}
  \frac{\partial \tau(\gamma (\theta, \xi), x)}{\partial \gamma (\theta, \xi)}  \frac{\partial \gamma (\theta, \xi)}{\partial \xi} g (\theta) \mathd
  \theta \\
  = & \int \frac{\partial f (\tau(\gamma (\theta, \xi), x))}{\partial \tau(\gamma (\theta, \xi), x)} 
  \frac{\partial \tau(\gamma (\theta, \xi), x)}{\partial \gamma (\theta, \xi)}  \frac{\partial \gamma (\theta, \xi)}{\partial \theta} M (\theta, \xi) g (\theta) \mathd
  \theta \\
  = & \int \frac{\partial f (\tau(\gamma (\theta, \xi), x))}{\partial \theta} M (\theta, \xi) g (\theta) \mathd \theta. 
\end{split}
\end{equation}
For any $u\in R^d$ satisfying $\|u\| = 1$, we have
\begin{eqnarray}
    \left|\langle \nabla_{\xi} G (\gamma (\xi, x)), u \rangle\right| & = & \left|\int \frac{\partial f (\tau(\gamma (\theta, \xi), x))}{\partial \theta} M (\theta, \xi) u g (\theta) \mathd \theta \right|\\
    & \leqslant & M^*\max_{\|v\|=1}\left|\int \frac{\partial f (\tau(\gamma (\theta, \xi), x))}{\partial \theta} v g (\theta) \mathd \theta \right|\\
    & \overset{1}{=} & M^*\max_{\|v\|=1}\left|\int  f (\tau(\gamma (\theta, \xi),  x)) \frac{\partial g (\theta)}{\partial \theta}  v   \mathd \theta \right|,
\end{eqnarray}
here $M^* = \max_{\theta, \xi} \|M (\theta, \xi)\|$ and equality 1 holds because of partial integral. Moreover, we assume that $g(\theta) = \exp(-\psi(\theta))$ and can further prove that
\begin{eqnarray}
    \left|\langle \nabla_{\xi} G (\gamma (\xi, x)), u \rangle\right| & \leqslant & M^*\max_{\|v\|=1}\left|\int  f (\tau(\gamma (\theta, \xi), x)) \frac{\partial g (\theta)}{\partial \theta}  v   \mathd \theta \right|\\
    & = & M^*\max_{\|v\|=1}\left|\int  f (\tau(\gamma (\theta, \xi), x)) g(\theta) \nabla \psi(\theta)  v   \mathd \theta \right|\\
    & = & M^*\max_{\|v\|=1}\left|\mathbb{E}_{\theta\sim g}\left[  f (\tau(\gamma (\theta, \xi), x))  \langle\nabla \psi(\theta),  v\rangle \right]   \right| \\
    & \leqslant & M^\ast \max_{\|v\|=1} \sup_{\widehat{f :} \hat{G}
    (\tau(\xi, x)) = G(\tau(\xi, x))} \mathbb{E}_{\theta \sim g} [\hat{f} (\tau(\gamma(\theta,\xi),x))
    \langle  \psi (\theta), u \rangle].
\end{eqnarray}
Similar with the proof of Theorem \ref{thm-4}, the optimal $\hat{f}$ achieves at
\begin{equation}
    \hat{f} (\tau(\gamma(\theta, \xi),x)) = \left\{\begin{array}{l}
      1,\quad \tmop{if} \langle \psi (\theta), u \rangle > \varphi^{- 1}_u (G (\tau(\xi,x)))),
      \\
      0,\quad \tmop{else}.
    \end{array}\right.
  \end{equation}
Then we have 
\begin{equation}
    \left|\langle \nabla_{\xi} G (\tau(\xi, x)), u \rangle\right| \leqslant \Phi(G(\tau(\xi,x))).
\end{equation}
Consider a path from $\zeta_t : [0, \| \delta \|] \rightarrow \mathbb{R}^d$ with $\zeta_0 = x$ , $\zeta_{\| \delta \|} = \tau(\xi,x)$ and $\zeta_t' =
\frac{\delta}{\| \delta \|}$, we have
\begin{equation}
    \frac{\mathd G (\xi_t)}{\mathd t} = \langle \nabla G (\xi_t), u \rangle
    \leqslant \Phi (G (\xi_t)) .
\end{equation}
The last part of proof is the same with Theorem \ref{thm-4}.
\end{proof}

\subsection{The proof of Theorem \ref{thm-2}}
\label{app_thm_2}
In this part, we will prove Theorem \ref{thm-2}, which is the main result in this paper.
\begin{reptheorem}{thm-2}
  Suppose $f (x)$ is any classifier and $\tilde{G} (\tilde{x})$ is the smoothed
  classifier defined in Eq.~(\ref{2}), if there exists $\underline{p_A}$ ,
  $\overline{p_B}$ that satisfies that
  \[ \tilde{G} (\tilde{x})_A \geqslant \underline{p_A} \geqslant
     \overline{p_B} \geqslant \tilde{G} (\tilde{x})_B, \]
  then $y_A = \margmax_{i \in \mathcal{Y}} \tilde{G} (\tilde{\tau}
  (\tilde{\xi}, \tilde{x}))_i$ for any $\| \xi \|_2 \leqslant R$ where
  \begin{equation}
     R = \frac{1}{ 2M^*} \int_{\overline{p_B}}^{\underline{p_A}} \frac{1}{\Phi(p)}\mathd p ,
  \end{equation}
  and the coefficient $M^*$ is defined as
  \begin{equation}
    M^{\ast} = \max_{\xi, \theta \in P} \sqrt{1 +
     \left\| \frac{\partial F_{2} (y_{\xi})}{\partial
    \xi} - \frac{\partial F_1 (\theta)}{\partial \theta} \right\|_2^2}.
    \label{rep-3}
  \end{equation}
\end{reptheorem}
\begin{proof}
WLOG, we prove it for binary cases where $f(\cdot):\mathbb{R}^n\to [0,1]$. First, we will calculate the gradient of $\tilde{G} (\tilde{\tau} (\tilde{\xi}, \tilde{x}))$ to $\tilde{\xi}$ as
\begin{equation}
  \nabla_{\tilde{\xi}} \tilde{G} (\tilde{\tau} (\tilde{\xi}, \tilde{x})) =
  \nabla_{\tilde{\xi}} \mathbb{E}_{\tilde{\theta} \sim \tilde{g}
  (\tilde{\theta})} [\tilde{f} (\tilde{\tau} (\tilde{\theta}, \tilde{\tau}
  (\tilde{\xi}, \tilde{x})))] .
\end{equation}
We expand the expectation into integral and use chain rule to see that
\begin{equation}
  \nabla_{\tilde{\xi}} \tilde{G} (\tilde{\tau} (\tilde{\xi}, \tilde{x})) =
  \int_{\mathbb{R}^{n + d}} \frac{\partial \tilde{f} (\tilde{\tau}
  (\tilde{\theta}, \tilde{y}_{\xi}))}{\partial \tilde{\tau} (\tilde{\theta},
  \tilde{y}_{\xi})}  \frac{\partial \tilde{\tau} (\tilde{\theta},
  \tilde{y}_{\xi})}{\partial \tilde{y}_{\xi}}  \frac{\partial
  \tilde{y}_{\xi}}{\partial \tilde{\xi}} \tilde{g} (\tilde{\theta}) \mathd
  \tilde{\theta} .
\end{equation}
Similar to previous proof, the key step is to eliminate the gradient of $\frac{\partial \tilde{f}
(\tilde{\tau} (\tilde{\theta}, \tilde{y}_{\xi}))}{\partial \tilde{\tau}
(\tilde{\theta}, \tilde{y}_{\xi})}$ and replace it with $\frac{\partial
\tilde{f} (\tilde{\tau} (\tilde{\theta}, \tilde{y}_{\xi}))}{\partial
\tilde{\theta}}$. Since
\begin{equation}
    \frac{\partial \tilde{\tau} (\tilde{\theta},
  \tilde{y}_{\xi})}{\partial \tilde{y}_{\xi}} = \begin{bmatrix}
    \sigma'_1 (z'_{\theta}) & \\
    & \sigma_2' (z_{\theta})
  \end{bmatrix} \begin{bmatrix}
    F'_{21} (y'_{\xi}) & \\
    & F_{22}' (y_{\xi})
  \end{bmatrix},
\end{equation}

\begin{equation}
    \frac{\partial
  \tilde{y}_{\xi}}{\partial \tilde{\xi}} = \begin{bmatrix}
    H'_1 (z_{\xi}') & \\
    & H_2' (z_{\xi})
  \end{bmatrix} \begin{bmatrix}
    I_d & \\
    F_1' (\xi) & I_n
  \end{bmatrix}.
\end{equation}
We have
\begin{eqnarray}
  \nabla_{\tilde{\xi}} \tilde{G} (\tilde{\tau} (\tilde{\xi}, \tilde{x})) & = & \int_{\mathbb{R}^{n +
  d}} \frac{\partial \tilde{F} (\tilde{\tau} (\tilde{\theta},
  \tilde{y}_{\xi}))}{\partial \tilde{\tau} (\tilde{\theta}, \tilde{y}_{\xi})}
  \begin{bmatrix}
    H'_1 (z'_{\theta}) & \\
    & H_2' (z_{\theta})
  \end{bmatrix} \begin{bmatrix}
    F'_{21} (y'_{\xi}) & \\
    & F_{22}' (y_{\xi})
  \end{bmatrix} \nonumber\\
  &  & \begin{bmatrix}
    H'_1 (z_{\xi}') & \\
    & H_2' (z_{\xi})
  \end{bmatrix} \begin{bmatrix}
    I_d & \\
    F_1' (\xi) & I_n
  \end{bmatrix} \tilde{g} (\tilde{\theta}) \mathd \tilde{\theta} \\
  & = & \int_{\mathbb{R}^{n + d}} \frac{\partial \tilde{F}
  (\tilde{\tau} (\tilde{\theta}, \tilde{y}_{\xi}))}{\partial \tilde{\tau}
  (\tilde{\theta}, \tilde{y}_{\xi})}  \begin{bmatrix}
    H'_1 (z'_{\theta}) & \\
    & H_2' (z_{\theta})
  \end{bmatrix} \begin{bmatrix}
    I_d & \\
    F_1' (\theta) & I_n
  \end{bmatrix} \begin{bmatrix}
    I_d & \\
    - F_1' (\theta) & I_n
  \end{bmatrix} \nonumber\\
  &  & \begin{bmatrix}
    F'_{21} (y'_{\xi}) & \\
    & F_{22}' (y_{\xi})
  \end{bmatrix} \begin{bmatrix}
    H'_1 (z_{\xi}') & \\
    & H_2' (z_{\xi})
  \end{bmatrix} \begin{bmatrix}
    I_d & \\
    F_1' (\xi) & I_n
  \end{bmatrix} \tilde{g} (\tilde{\theta}) \mathd \tilde{\theta} \\
  & = & \int_{\mathbb{R}^{n + d}} \frac{\partial \tilde{F} (\tilde{\tau}
  (\tilde{\theta}, \tilde{y}_{\xi}))}{\partial \tilde{\theta}}
  \begin{bmatrix}
    I_d & \\
    - F_1' (\theta) & I_n
  \end{bmatrix} \begin{bmatrix}
    F'_{21} (y'_{\xi}) & \\
    & F_{22}' (y_{\xi})
  \end{bmatrix} \nonumber\\
  &  & \begin{bmatrix}
    H'_1 (z_{\xi}') & \\
    & H_2' (z_{\xi})
  \end{bmatrix} \begin{bmatrix}
    I_d & \\
    F_1' (\xi) & I_n
  \end{bmatrix} \tilde{g} (\tilde{\theta}) \mathd \tilde{\theta} \\
  & = & \int_{\mathbb{R}^{n + d}} \frac{\partial \tilde{F} (\tilde{\tau}
  (\tilde{\theta}, \tilde{y}_{\xi}))}{\partial \tilde{\theta}}
  \tilde{M} (\tilde{\xi},  \tilde{\theta}) g (\tilde{\theta}) \mathd \tilde{\theta}, 
\end{eqnarray}
here we define
\begin{equation}
  \tilde{M} (\tilde{\xi},  \tilde{\theta}) \triangleq \begin{bmatrix}
    I_d & \\
    - F_1' (\theta) & I_n
  \end{bmatrix} \begin{bmatrix}
    F'_{21} (y'_{\xi}) & \\
    & F_{22}' (y_{\xi})
  \end{bmatrix} \begin{bmatrix}
   H'_1 (z_{\xi}') & \\
    & H_2' (z_{\xi})
  \end{bmatrix} \begin{bmatrix}
    I_d & \\
    F_1' (\xi) & I_n
  \end{bmatrix}.
\end{equation}
We consider the unit enlargement, which means that
\begin{equation}
  H_1 (z') = z', F_{21} (x') = x',
\end{equation}
thus we have
\begin{equation}
  \tilde{M} (\tilde{\xi},  \tilde{\theta}) = \begin{bmatrix}
    I_d & O_{d \times n}\\
    F_{21}' (y_{\xi}) - F_1' (\theta) & F_{22}' (y_{\xi}) H_2' (z_{\xi})
  \end{bmatrix}.
\end{equation}
Since $\theta'$ is the virtual parameter introduced, which can be taken as 0 in case of actual disturbance. Thus we only need to consider the projection of $\nabla_{\tilde{\xi}} \tilde{G}
(\tilde{y}_{\xi})$ in the space of $\xi$, i.e., we set
\begin{equation}
  \tilde{u} = \left(\begin{array}{c}
    u\\
    O_{n \times 1}
  \end{array}\right),
\end{equation}
here $u \in \mathbb{R}^d$ satisfying $\| u \| = 1$. Assume
\begin{equation}
\begin{split}
  \tilde{g} (\tilde{\theta})& = \exp (- \tilde{\psi} (\tilde{\theta})), \\
  \frac{\partial \tilde{g} (\tilde{\theta})}{\partial \tilde{\theta}}& = -
  \tilde{g} (\tilde{\theta})  \nabla \tilde{\psi} (\tilde{\theta}).
\end{split}
\end{equation}
And we have
\begin{eqnarray}
 \left|\langle \nabla_{\tilde{\xi}} \tilde{G} (\tilde{y}_{\xi}), \tilde{u} \rangle \right|
  &=&  \left|  \int_{\mathbb{R}^{n + d}} \frac{\partial \tilde{f} (\tilde{\tau}
  (\tilde{\theta}, \tilde{y}_{\xi}))}{\partial \tilde{\theta}}
  \tilde{M} (\tilde{\xi},  \tilde{\theta})\tilde{u} \tilde{g} (\tilde{\theta}) \mathd \tilde{\theta} \right|\\
  &=& \left|\int_{\mathbb{R}^{n + d}} \frac{\partial \tilde{f} (\tilde{\tau}
  (\tilde{\theta}, \tilde{y}_{\xi}))}{\partial \tilde{\theta}} \begin{bmatrix}
    I_d,& O_{d\times n}\\
    F_{22}' (y_{\xi}) - F_1' (\theta),& O_{n\times n}
  \end{bmatrix}
   \tilde{u} \tilde{g} (\tilde{\theta}) \mathd \tilde{\theta}\right|\\
   &=& \left|\int_{\mathbb{R}^{n + d}} \frac{\partial \tilde{f} (\tilde{\tau}
  (\tilde{\theta}, \tilde{y}_{\xi}))}{\partial \tilde{\theta}} M(\xi,\theta)
   \tilde{u} \tilde{g} (\tilde{\theta}) \mathd \tilde{\theta}\right|\\
    &\leqslant& M^* \max_{\|\tilde{v}\|_2=1} \left|\int_{\mathbb{R}^{n + d}} \frac{\partial \tilde{f} (\tilde{\tau}
  (\tilde{\theta}, \tilde{y}_{\xi}))}{\partial \tilde{\theta}}
   \tilde{v} \tilde{g} (\tilde{\theta}) \mathd \tilde{\theta}\right|\\
  &\overset{1}{=}& M^* \max_{\|\tilde{v}\|_2=1} \left|\int_{\mathbb{R}^{n + d}} \tilde{f} (\tilde{\tau}
  (\tilde{\theta}, \tilde{y}_{\xi})) \frac{\partial \tilde{g} (\tilde{\theta})}{\partial \tilde{\theta}}
   \tilde{v}  \mathd \tilde{\theta}\right|\\
   &=& M^* \max_{\|\tilde{v}\|_2=1} \left|\int_{\mathbb{R}^{n + d}} \tilde{f} (\tilde{\tau}
  (\tilde{\theta}, \tilde{y}_{\xi})) \tilde{g} (\tilde{\theta}) \nabla \tilde{\psi} (\tilde{\theta})
   \tilde{v}  \mathd \tilde{\theta}\right|\\
   &=& M^*\max_{\|\tilde{v}\|_2=1}\left|\mathbb{E}_{\theta\sim g}\left[ \tilde{f}(\tilde{\tau}(\tilde{\theta},\tilde{\tau}(\tilde{\xi}, x)))  \langle\nabla \tilde{\psi}(\tilde{\theta}),  \tilde{v}\rangle \right]   \right|, 
\end{eqnarray}
here equality 1 holds because of partial integral. Moreover, we can bound the right hand side by 
\begin{equation}
    \left|\langle \nabla_{\tilde{\xi}} \tilde{G} (\tilde{y}_{\xi}), \tilde{u} \rangle \right| 
      \leqslant \sup_{\widehat{f :} \hat{G}
    (\tau(\xi, x)) = G(\tau(\xi, x))} M^*\max_{\|\tilde{v}\|_2=1} \  \left|\mathbb{E}_{\theta\sim g}\left[ \hat{f}(\tilde{\tau}(\tilde{\theta},\tilde{\tau}(\tilde{\xi}, x)))  \langle\nabla \tilde{\psi}(\tilde{\theta}),  \tilde{v}\rangle \right]   \right|,
\end{equation}
and the optimal $\hat{f}$ is
\begin{equation}
    \hat{f} (\tilde{\tau}(\tilde{\theta},\tilde{\tau}(\tilde{\xi},\tilde{x})) = \left\{\begin{array}{l}
      1, \tmop{if} \langle \tilde{\psi} (\tilde{\theta}), \tilde{u} \rangle > \varphi^{- 1}_u (\tilde{G} (\tilde{\tau}(\tilde{\xi},\tilde{x}))))
      \\
      0, \tmop{else}
    \end{array},\right.
\end{equation}
here
\begin{eqnarray}
    M(\xi,\theta) = \begin{bmatrix}
    I_d,& O_{d\times n}\\
    F_{22}' (y_{\xi}) - F_1' (\theta),& O_{n\times n},
  \end{bmatrix},
\end{eqnarray}
and $M^*$ is
\begin{equation}
\begin{split}
  M^* &= \max_{\xi, \theta \in P} \left\| \begin{bmatrix}
    I_d,& O_{d\times n}\\
    F_{22}' (y_{\xi}) - F_1' (\theta),& O_{n\times n}
  \end{bmatrix}\right\|_2\\ 
  &= \max_{\xi, \theta \in P} \left\| \begin{bmatrix}
    I_d\\
    \frac{\partial F_{22} (y_{\xi})}{\partial \xi} - \frac{\partial F_1
    (\theta)}{\partial \theta}
  \end{bmatrix}\right\|_2\\
  &= \max_{\xi, \theta \in P} \sqrt{1 +
   \left\| \left( \frac{\partial F_{22}
  (y_{\xi})}{\partial \xi} - \frac{\partial F_1 (\theta)}{\partial \theta}
  \right) \right\|_2^2}.
\end{split}
\end{equation}
Notice that here $F_{22}(\cdot)$ is the same as $F_2(\cdot)$ in the main text.
Then we could apply the techniques used in Theorem \ref{thm-1} and Theorem \ref{thm-4} here and have:
\begin{equation}
   R = \frac{1}{ 2M^*} \int_{\overline{p_B}}^{\underline{p_A}} \frac{1}{\Phi(p)}\mathd p.
\end{equation}
Thus we have proven this theorem.
\end{proof}
\subsection{The Proof of Theorem \ref{thm-3}}
\label{app_thm_3}
In this part, we will prove Theorem~\ref{thm-3}.
\begin{reptheorem}{thm-3}
   Suppose the simulation of the semantic transformation has a small enough error
  \[ \| \tilde{\tau} (\tilde{\xi}, \tilde{x}) - \overline{\tau} (\tilde{\xi}, \tilde{x})
     \|_2 < \varepsilon, \]
    where $\overline{\tau} (\tilde{\xi}, \tilde{x})$ is the real semantic transformation.
  Then there exists a constant ratio 
  \[A=A(\|F_1'(\tilde{\xi})\|_2, \|F_2'(\tilde{y}_\xi)\|_2,\|F_2'(\tilde{z}_\xi)\|_2)>0,\]
  which does not depend on the target classifier. The certified radius for the real semantic transformation satisfies: 
  \begin{equation}
      R_r > R (1 - A \varepsilon),
  \end{equation} 
  where $R$ is the certified radius in Eq.~\eqref{eq:the2-r} for the surrogate neural network in Theorem \ref{thm-2}.
  
\end{reptheorem}

\begin{proof}
We set
\begin{equation}
  \tilde{u} = \left(\begin{array}{c}
    u\\
    O_{n \times 1}
  \end{array}\right),
\end{equation}
here $u \in \mathbb{R}^d$ satisfying $\| u \| = 1$. Then we have
\begin{equation}
\begin{split}
    \left\langle \nabla_{\tilde{\xi}}\tilde{G}(\tilde{\tau}(\tilde{\xi},\tilde{x} )) - \nabla_{\tilde{\xi}}\tilde{G}(\bar{\tau}(\tilde{\xi},\tilde{x} )), \tilde{u}\right\rangle &= \int \left( \frac{\partial\tilde{f}(\tilde{\tau}(\tilde{\theta}, \tilde{\tau}(\tilde{\xi}, \tilde{x})))}{\partial \tilde{\xi}} - \frac{\partial\tilde{f}(\tilde{\tau}(\tilde{\theta}, \bar{\tau}(\tilde{\xi}, \tilde{x})))}{\partial \tilde{\xi}}\right) \tilde{u} \tilde{g}(\tilde{\theta})\mathd \tilde{\theta}\\
    &= \int \left(  \tilde{\tau}(\tilde{\xi}, \tilde{x}) - \bar{\tau}(\tilde{\xi}, \tilde{x})\right) \frac{\partial^2 \tilde{f}(\tilde{\tau}(\tilde{\theta}, \hat{\tau}(\tilde{\xi}, \tilde{x})))}{\partial \tilde{\xi}\partial \hat{\tau}} \tilde{u} \tilde{g}(\tilde{\theta})\mathd \tilde{\theta}.
\end{split}
\end{equation}
Set $L = \frac{\partial^2 \tilde{f}(\tilde{\tau}(\tilde{\theta}, \bar{\tau}(\tilde{\xi}, \tilde{x})))}{\partial \tilde{\xi}\partial \hat{\tau}}$, and we have
\begin{equation}
\begin{split}
    L &= \frac{\partial }{\partial \tilde{\xi}}\left(\frac{\partial \tilde{f}(\tilde{\tau}(\tilde{\theta}, \hat{\tau}(\tilde{\xi}, \tilde{x})))}{\partial \hat{\tau}}\right) = \frac{\partial }{\partial \tilde{\xi}}\left(\frac{\partial \tilde{f}(\tilde{\tau}(\tilde{\theta}, \hat{\tau}(\tilde{\xi}, \tilde{x})))}{\partial \tilde{\tau}(\tilde{\theta}, \hat{\tau}(\tilde{\xi}, \tilde{x}))} \frac{\partial \tilde{\tau}(\tilde{\theta}, \hat{\tau}(\tilde{\xi}, \tilde{x}))}{\partial \hat{\tau}}\right).
\end{split}
\end{equation}
Set $y_{\tilde{\xi}} = \tilde{\tau}(\tilde{\theta}, \hat{\tau}(\tilde{\xi}, \tilde{x}))$, and we have:
\begin{equation}
\begin{split}
    \frac{\partial \tilde{\tau}(\tilde{\theta}, \hat{\tau}(\tilde{\xi}, \tilde{x}))}{\partial \hat{\tau}} &= \frac{\partial \tilde{H} (\tilde{F_1} (\tilde{\theta}) + \tilde{F_2} (\hat{\tau}(\tilde{\xi}, \tilde{x})))}{\partial \hat{\tau}}\\
    &= \tilde{H}' \left(\tilde{F_1} (\tilde{\theta}) + \tilde{F_2} (\hat{\tau}(\tilde{\xi}, \tilde{x}))\right) \frac{\partial \tilde{F_2} (\hat{\tau}(\tilde{\xi}, \tilde{x}))}{\partial \hat{\tau}}\\
    &= \tilde{H}' \left(\tilde{F_1} (\tilde{\theta}) + \tilde{F_2} (\hat{\tau}(\tilde{\xi}, \tilde{x}))\right) \begin{bmatrix}
    I_d & \\
    F_1' (\theta) & I_n
  \end{bmatrix} \begin{bmatrix}
    I_d & \\
    - F_1' (\theta) & I_n
  \end{bmatrix} \frac{\partial \tilde{F_2} (\hat{\tau}(\tilde{\xi}, \tilde{x}))}{\partial \hat{\tau}}\\
  &= \frac{\partial \tilde{\tau}(\tilde{\theta}, \hat{\tau}(\tilde{\xi}, \tilde{x}))}{\partial \tilde{\theta}} \begin{bmatrix}
    I_d & \\
    - F_1' (\theta) & I_n
  \end{bmatrix} \tilde{F_2}' (\hat{\tau}(\tilde{\xi}, \tilde{x}))\\
  &= \frac{\partial \tilde{\tau}(\tilde{\theta}, \hat{\tau}(\tilde{\xi}, \tilde{x}))}{\partial \tilde{\theta}} A_1,
\end{split}
\end{equation}

here $A_1 = \begin{bmatrix}
    I_d & \\
    - F_1' (\theta) & I_n
  \end{bmatrix} \tilde{F_2}' (\hat{\tau}(\tilde{\xi}, \tilde{x}))$. By the proof above, we have $\frac{\partial }{\partial \tilde{\xi}} = \frac{\partial}{\partial \tilde{\theta}} A_2$, thus we have
\begin{equation}
\begin{split}
    L &= \frac{\partial }{\partial \tilde{\xi}}\left(\frac{\partial \tilde{f}(\tilde{\tau}(\tilde{\theta}, \hat{\tau}(\tilde{\xi}, \tilde{x})))}{\partial \tilde{\tau}(\tilde{\theta}, \hat{\tau}(\tilde{\xi}, \tilde{x}))} \frac{\partial \tilde{\tau}(\tilde{\theta}, \hat{\tau}(\tilde{\xi}, \tilde{x}))}{\partial \hat{\tau}}\right)\\
    &= \frac{\partial }{\partial \tilde{\theta}}\left(\frac{\partial \tilde{f}(\tilde{\tau}(\tilde{\theta}, \hat{\tau}(\tilde{\xi}, \tilde{x})))}{\partial \tilde{\tau}(\tilde{\theta}, \hat{\tau}(\tilde{\xi}, \tilde{x}))} \frac{\partial \tilde{\tau}(\tilde{\theta}, \hat{\tau}(\tilde{\xi}, \tilde{x}))}{\partial \tilde{\theta}} A_1\right) A_2\\
    &= \frac{\partial }{\partial \tilde{\theta}}\left(\frac{\partial \tilde{f}(\tilde{\tau}(\tilde{\theta}, \hat{\tau}(\tilde{\xi}, \tilde{x})))}{\partial \tilde{\theta}} A_1\right) A_2.
\end{split}
\end{equation}
Furthermore, we can prove that
\begin{equation}
\begin{split}
    &\left\langle \nabla_{\tilde{\xi}}\tilde{G}(\tilde{\tau}(\tilde{\xi},\tilde{x} )) - \nabla_{\tilde{\xi}}\tilde{G}(\bar{\tau}(\tilde{\xi},\tilde{x} )), \tilde{u}\right\rangle\\
    =& \int \left(  \tilde{\tau}(\tilde{\xi}, \tilde{x}) - \bar{\tau}(\tilde{\xi}, \tilde{x})\right) \frac{\partial^2 \tilde{f}(\tilde{\tau}(\tilde{\theta}, \hat{\tau}(\tilde{\xi}, \tilde{x})))}{\partial \tilde{\xi}\partial \hat{\tau}} \tilde{u} \tilde{g}(\tilde{\theta})\mathd \tilde{\theta}\\
    =& \int \left(  \tilde{\tau}(\tilde{\xi}, \tilde{x}) - \bar{\tau}(\tilde{\xi}, \tilde{x})\right) \frac{\partial }{\partial \tilde{\theta}}\left(\frac{\partial \tilde{f}(\tilde{\tau}(\tilde{\theta}, \hat{\tau}(\tilde{\xi}, \tilde{x})))}{\partial \tilde{\theta}} A_1\right) A_2 \tilde{u} \tilde{g}(\tilde{\theta})\mathd \tilde{\theta} \\
    \leqslant& \tilde{A} \epsilon \left| \int \tilde{f}(\tilde{\tau}(\tilde{\theta}, \hat{\tau}(\tilde{\xi}, \tilde{x}))) \tilde{g}(\tilde{\theta}) \langle \nabla \tilde{\psi}(\tilde{\theta}), \tilde{u} \rangle \tmop{d}\tilde{\theta}\right|,
\end{split}
\end{equation}
where $\tilde{A}$ is a constant depending on $\|F_1'(\tilde{\xi})\|, \|F_2'(\tilde{y}_\xi)\|$ and $\|F_2'(\tilde{z}_\xi)\|$. Then there exists $A$ and we have 
\begin{equation}
    R_r \geqslant R(1-\epsilon A),
\end{equation}
where 
\[
     R = \frac{1}{ 2M^*} \int_{\overline{p_B}}^{\underline{p_A}} \frac{1}{\Phi(p)}\mathd p .
\]
Thus we have proven this theorem.
\end{proof}
\section{Implementation Details and Experimental Settings}
\label{appendix2}
\subsection{Practical Algorithms for Calculating $M^\ast$}

For resolvable transformations in Theorem \ref{thm-1}, the $M^\ast$ is defined as
\begin{equation}
    M^\ast = \max_{\xi,\theta \in P} \|M(\xi, \theta)\|.
\end{equation}
Since if we have verified that the semantic transformation is resolvable, most of time we have a closed form of $M^\ast$ like contrast/brightness transformation and we are able to calculate it analytically as shown in \citet{li2021tss}.

For non-resolvable transformations in Corollary \ref{cor-1}, $M^\ast$ is defined as 
\begin{equation}
 M^{\ast} = \max_{\xi \in P} \sqrt{\frac{1}{\sigma_1^2} + \frac{1}{\sigma_2^2}
    \left\| \frac{\partial F_{2} (y_{\xi})}{\partial \xi} - A_1 \right\|_2^2}.
\end{equation}
This ratio is similar with the Lipschitz bound for a semantic transformation in \citet{li2021tss}.For low dimensional semantic transformations, we are able to interpolate the domain to find a maximum $M^\ast$ and corresponding $\xi$. But this remains a challenge for high dimensional semantic transformations. Specifically, for a given $\xi$, we need to compute the norm of $\frac{\partial F_2(y_\xi)}{\partial \xi}-A_1$. The Jacobian matrix is $n \times n$. Caculating it requires $n$ times of backpropagation. Thus it is inefficient to store the matrix or directly compute its norm. To solve the problem, we noitice that 
\begin{equation}
    \left\|\frac{\partial F_2(y_\xi)}{\partial \xi}-A_1 \right\|_2^2 = \max_{\|u\|_2=1} \left\|\left(\frac{\partial F_2(y_\xi)}{\partial \xi} - A_1\right)u\right\|_2^2.
\end{equation}
And then we have
\begin{equation}
    \left(\frac{\partial F_2(y_\xi)}{\partial \xi}-A_1\right)^\top u = \frac{\partial}{\partial \xi} \left\langle F_2(y_\xi)-A_1, u\right\rangle.
\end{equation}
Since tranposing a matrix does not change its norm, we could calculate its norm by optimizing $u$ that,
\begin{equation}
    \max_{\|u\|_2=1} \left\|\frac{\partial}{\partial \xi} \langle F_2(y_\xi)-A_1, u\rangle\right\|_2^2.
\end{equation}
Using this formulation, we only need to multiply the output with an unit vector and perform one backpropagation. This is a simple convex optimization problem. Then we could use any iterative algorithm to find its solution which is very fast to compute. This trick is crucial and it makes the matrix norm computation to be scalable in practice. 

\subsection{Other Experimental Details}

Our GSmooth requires to train two neural networks. First we randomly generate corrupted images to train a image-to-image neural network. The training process of classifiers and certification for semantic transformations are done on 2080Ti GPUs.  We use a U-Net \citep{ronneberger2015u} for the surrogate model and replace all BatchNorm layers with GroupNorm \citep{wu2018group} since we might use the model in low bacthsize settings. The U-Net could be replace by other networks used in image segmentation or superresolution like Res-UNet \citep{diakogiannis2020resunet} or EDSR \citep{lim2017enhanced}. We use L1-loss to train the surrogate model which achieves better accuracy than others which is also reported \citep{lim2017enhanced}. 

After training a surrogate model to simulate the semantic transformation, we then train the base classifier for certification with a moderate data augmentation \citep{li2021tss, cohen2019certified} to ensure that training and testing of the classifier is performed on the same distribution. There are two types of data augmentation, one is the semantic transformation  and the other is the augmented noise introduced only in our work. Data augmentation based semantic transformation could be done using both the surrogate model or the raw semantic transformation. We can only use the surrogate model to add the augmented noise because this noise is a type of semantic noise in the layers of surrogate model. In our experiments the standard deviation of the augmented noise is chosen from $0.1\sim0.4$ depending on the performance. The basic network architectures for these datasets are kept the same with \citet{li2021tss}. On CIFAR-100 daatsets, we use a PreResNet \citep{he2016identity} re-implement the method by \citet{li2021tss}. We also adopt the progressive sampling trick mentioned in TSS \citep{li2021tss} which is useful to reduce computational cost and certify larger radius. The details could also be found in \citet{li2021tss}.

\end{document}